\documentclass{article}

\usepackage{microtype}
\usepackage{graphicx}
\usepackage{subcaption}
\usepackage{booktabs} 

\usepackage{hyperref}

\usepackage[preprint]{icml2026}

\usepackage{amsmath}
\usepackage{amssymb}
\usepackage{mathtools}
\usepackage{amsthm}

\usepackage[capitalize,noabbrev]{cleveref}

\theoremstyle{plain}
\newtheorem{theorem}{Theorem}[section]

\newtheorem{corollary}[theorem]{Corollary}
\theoremstyle{definition}

\theoremstyle{remark}
\newtheorem{remark}[theorem]{Remark}

\usepackage[textsize=tiny]{todonotes}

\icmltitlerunning{LUCID-SAE: Learning Unified Vision-Language Sparse Codes for Interpretable Concept Discovery}

\begin{document}

\twocolumn[{
  \icmltitle{LUCID-SAE: Learning Unified Vision-Language Sparse Codes for Interpretable Concept Discovery}



  \icmlsetsymbol{equal}{*}

  \begin{icmlauthorlist}
    \icmlauthor{Difei Gu}{rutgers}
    \icmlauthor{Yunhe Gao}{stanford}
    \icmlauthor{Gerasimos Chatzoudis}{rutgers}
    \icmlauthor{Zihan Dong}{rutgers}
    \icmlauthor{Guoning Zhang}{rutgers}
    \icmlauthor{Bangwei Guo}{rutgers}
    \icmlauthor{Yang Zhou}{rutgers}
    \icmlauthor{Mu Zhou}{rutgers}
    \icmlauthor{Dimitris Metaxas}{rutgers}
  \end{icmlauthorlist}

  \icmlaffiliation{rutgers}{Rutgers University}
  \icmlaffiliation{stanford}{Stanford University}

  \icmlcorrespondingauthor{Dimitris Metaxas}{dnm@cs.rutgers.edu}

  \vspace{5mm} 

\begin{center}
    \includegraphics[width=1\textwidth, height=5cm]{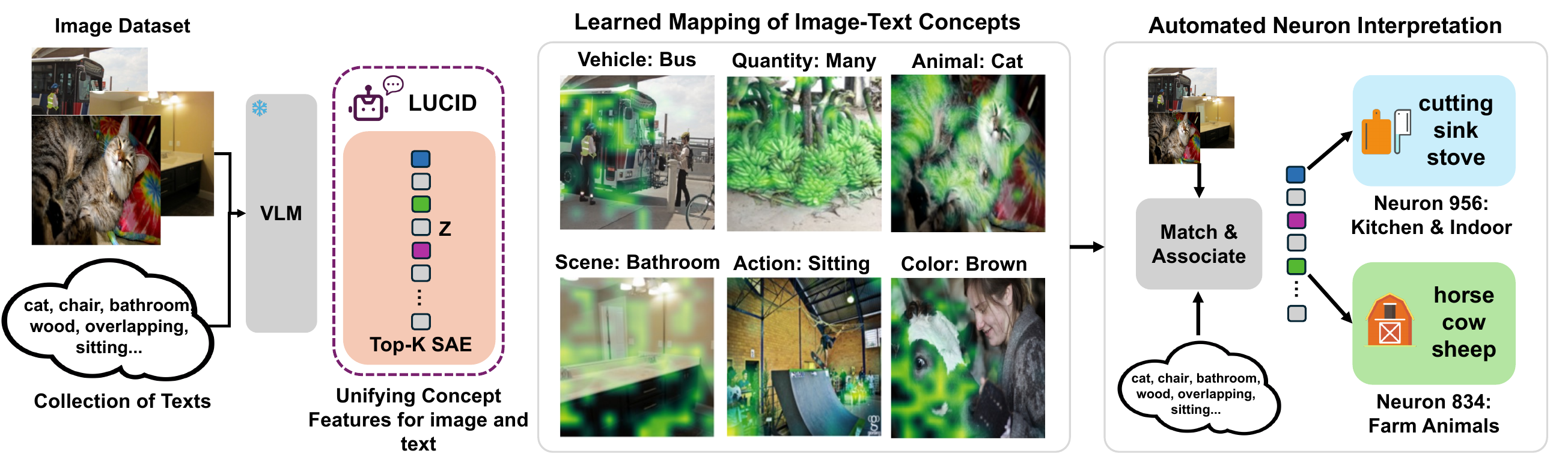}
    \captionof{figure}[width=0.9\textwidth]{\textbf{Overview of LUCID.} LUCID processes image patches and text tokens through a vision-language model (VLM) to learn a uniform sparse dictionary via a Top-K SAE. The unified latent space captures interpretable concept features spanning multiple semantic categories such as objects (\emph{Bus}, \emph{Cat}), scenes (\emph{Bathroom}), actions (\emph{Sitting}), visual attributes (\emph{Brown}), and quantities (\emph{Many}), enabling cross-modal feature alignment and automatic interpretation of image concepts.}
    \label{figure:method_overview}
\end{center}

  \icmlkeywords{Machine Learning, ICML}

  \vskip 0.2in
}]



\printAffiliationsAndNotice{}  

\begin{abstract}
Sparse autoencoders (SAEs) offer a natural path toward comparable explanations across different representation spaces. However, current SAEs are trained per modality, producing dictionaries whose features are not directly understandable and whose explanations do not transfer across domains. In this study, we introduce \textbf{LUCID} (\textbf{L}earning \textbf{U}nified vision-language sparse \textbf{C}odes for \textbf{I}nterpretable concept \textbf{D}iscovery), a unified vision-language sparse autoencoder that learns a shared latent dictionary for image patch and text token representations, while reserving private capacity for modality-specific details. We achieve feature alignment by coupling the shared codes with a learned optimal transport matching objective without the need of labeling. LUCID yields interpretable shared features that support patch-level grounding, establish cross-modal neuron correspondence, and enhance robustness against the concept clustering problem in similarity-based evaluation. Leveraging the alignment properties, we develop an automated dictionary interpretation pipeline based on term clustering without manual observations. Our analysis reveals that LUCID's shared features capture diverse semantic categories beyond objects, including actions, attributes, and abstract concepts, demonstrating a comprehensive approach to interpretable multimodal representations.

\end{abstract}


\section{Introduction}

Modern representation learning models encode rich semantic structure in high-dimensional activations~\cite{zhai2023sigmoid, brown2020language, bengio2012representation, vaswani2017attention}. Yet turning these activations into explanations that are stable, comparable, and transferable remains challenging. Growing efforts~\cite{bricken2023monosemanticity, cunningham2023sparse} have explored disentangled feature codes that decompose dense activations into separable concepts. Among these, sparse autoencoders (SAEs)~\cite{lim2024sparse, stevens2025sparse} represents as a key paradigm by mapping activations to sparse codes over learned dictionary elements. These yield interpretable units that can be inspected, visualized, and leveraged for feature-level analyses. This feature-centric perspective has established SAEs as increasingly valuable tools for interpretation, evaluation, and safety-oriented understanding of learned representations~\cite{cunningham2023sparse,bricken2023monosemanticity, pach2025sparse, ferrando2024know, yinconstrain}.

A key limitation of SAEs is that they are always trained independently for each modality~\cite{fry2024towards, thasarathan2025universal, templeton2024scaling}. Consequently, concepts learned in one modality lack the guaranteed correspondence to analogous concepts in the another. This absence of cross-space comparability impedes the construction of consistent explanations across diverse inputs and domains. It also fundamentally restricts downstream applications requiring shared semantics, such as grounding language and vision evidence within a unified set of interpretable concepts. Without an explicit alignment, independently trained SAEs tend to learn idiosyncratic, model-specific bases. In this study, we ask a fundamental question: \textbf{Can we learn a single set of sparse codes whose units are directly comparable across vision and language spaces?}

Achieving this objective presents significant challenges because representation spaces differ substantially in their distributional properties, dimensionality, and tokenization granularity. While training separate SAEs preserves reconstruction fidelity, it sacrifices the alignment necessary for cross-model interpretation. At the heart of this challenge, the core trade-off centers on: (i) a shared set of semantic units that enables comparability and transfer, and (ii) sufficient flexibility to capture modality-specific factors without constraining them to the common representation.

We introduce LUCID as a unified sparse autoencoder that learns a shared latent code used to encode VLM~\cite{zhai2023sigmoid, radford2021learning}. Specifically, LUCID disentangles image patch and text token activations and produces sparse codes within a common latent space, enabling direct concept-level correspondence across modalities. Additionally, LUCID reserves private capacity for each space, allowing key representation of space-specific information that need not be shared. Training combines reconstruction objectives for each space with an alignment signal that encourages paired inputs (e.g., image-caption pairs) to activate consistent shared concepts. As a result, this yields interpretable features that are simultaneously sparse and cross-space comparable without a need of labeling.

Our major contributions are:
\begin{itemize}
    \item \textbf{Unified multimodal sparse coding.} We introduce LUCID, a sparse autoencoder framework that learns a shared dictionary across vision and language through shared-private decomposition and optimal transport-based alignment.
    
    \item \textbf{Patch-level grounding via guided optimal transport.} We develop GCMT (Guided Cross-Modal Transport), which combines entropy-regularized transport with structural and contextual guidance to establish the fine-grained alignment between image patches and text tokens. 
    
    \item \textbf{Automated interpretation and empirical analysis.} We demonstrate an automatic neuron interpretation pipeline through cross-modal concept clustering, revealing coherent semantic specializations without manual annotation. 

\end{itemize}

\section{Related Work}

\textbf{Model Interpretation with Concepts.} Concept Bottleneck Models (CBMs)~\cite{koh2020concept,rao2024discover,gao2024aligning} represent a seminal approach in this direction. They introduce an interpretable intermediate layer where predictions are explicitly mediated through predefined concepts. In the original CBM framework, models are trained to first predict a set of human-specified concepts (e.g., ``has wings," ``is blue") before making final task predictions. This creates a transparent decision pathway. The architecture enables practitioners to inspect which concepts contribute to each prediction and even intervene on concept activations to correct model behavior. However, CBMs face significant limitations. They require extensive concept annotations during training. The predefined concept set requires heavy manual design and may not align with the features the model naturally learns~\cite{havasi2022addressing}. As a result, this bottleneck constraint can compromise overall task performance. 

\textbf{Mechanisic Interpretability for Sparse Autoencoders.} Recent advances in mechanistic interpretability~\cite{elhage2021mathematical,bricken2023monosemanticity, cunningham2023sparse} have introduced sparse autoencoders (SAEs) as a powerful tool for decomposing neural network representations into interpretable components. SAEs address a critical challenge of polysemantic nature of neurons, where individual neurons often respond to multiple unrelated features due to superposition. The goal of SAEs is to learn an overcomplete dictionary~\cite{sharkey2022taking} of monosemantic features that linearly reconstruct model activations under sparsity constraints. This enables SAEs to effectively disentangle superposed representations, recovering interpretable feature directions from polysemantic neurons~\cite{elhage2022toy}. Beyond basic feature extraction, recent work has trained SAEs at various network depths~\cite{rajamanoharan2024improving} and leveraged SAE features for model steering~\cite{cho-hockenmaier-2025-toward}. However, current SAE faces several key limitations. First, the methodology has been predominantly applied to language models with limited exploration in vision domains~\cite{fry2024towards,thasarathan2025universal} and vision-language cross-modal concept alignment. Second, the features discovered by SAEs are highly fine-grained and task-specific, making them difficult to interpret through manual inspection alone~\cite{paulo2024automatically}. In particular, current efforts lack mechanisms for automatic vision concept interpretation. This dependency on manual inspection fundamentally limits the practical utility of SAEs as an interpretability tool in a multimodal setting.

\section{Method}
\textbf{Problem Formulation.} We study the problem of learning interpretable, sparse feature codes that are comparable across multiple representation spaces. Let \(\mathcal{M}\) denote a set of activation spaces from the VLM. For each \(m \in \mathcal{M}\), we assume access to a representation function \(f_{m}\) that maps an input \(x^{(m)}\) to a set of feature vectors \(H^{(m)}\):
\begin{equation*}
H^{(m)} = f_m(x^{(m)}) \in \mathbb{R}^{d_m},
f_m : \mathbb{R}^{T_m}\rightarrow \mathbb{R}^{d_m},
\end{equation*}

where \(T_m\) is the length of the token sequence of image patches or text tokens, and \(d_m\) is the feature dimension in space \(m\). In the case of Vision Language Models (VLMs), we provide dataset \(\mathcal{D}\) that contains paired images and texts samples across spaces:
\begin{equation*}
\mathcal{D} = {(x_n^{(\text{img})},x_n^\text{(txt)})}^N_{n=1},
\end{equation*}
where \(\text{img},\text{txt} \in \mathcal{M}\). Our goal is to find a training paradigm that learn a unified sparse autoencoder that produces shared sparse dictionary \(z\) whose coordinates have the same semantic meaning across spaces.



\begin{figure}[t]
\centering
\label{fig:all22}
\includegraphics[width=1\linewidth]{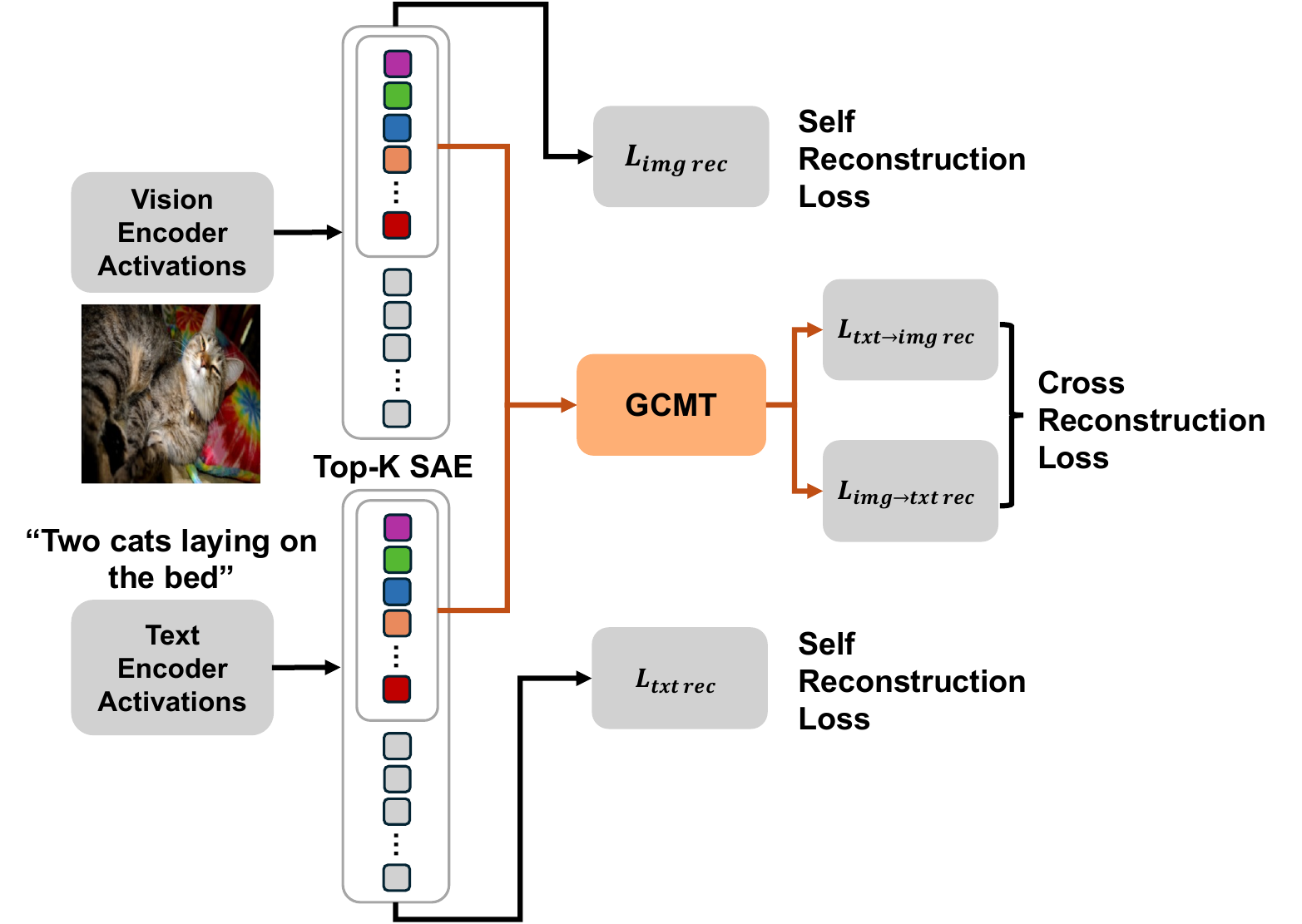}\\
\caption{Model structure of LUCID showing the high-level loss components. The model is optimized at two levels: self-reconstruction is done on the shared + private code space, while cross-reconstruction is done on the shared code space.}
\end{figure}

\begin{figure*}[t]
\centering

\makebox[\textwidth][c]{%
  \includegraphics[width=0.49\textwidth]{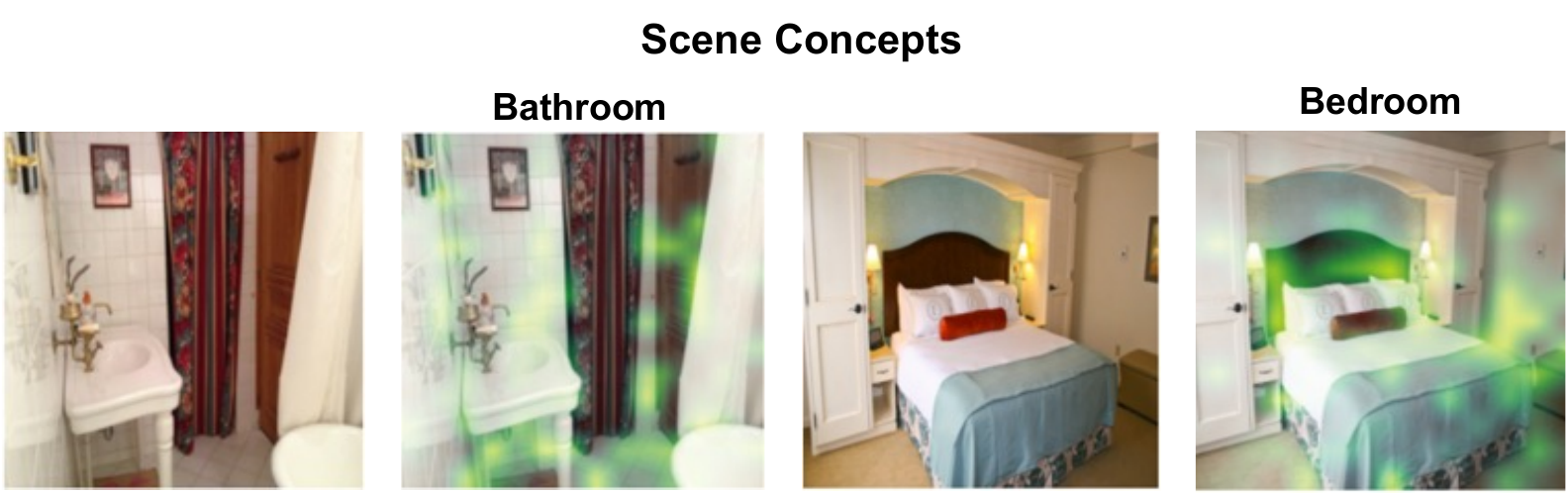}\hfill
  \includegraphics[width=0.49\textwidth]{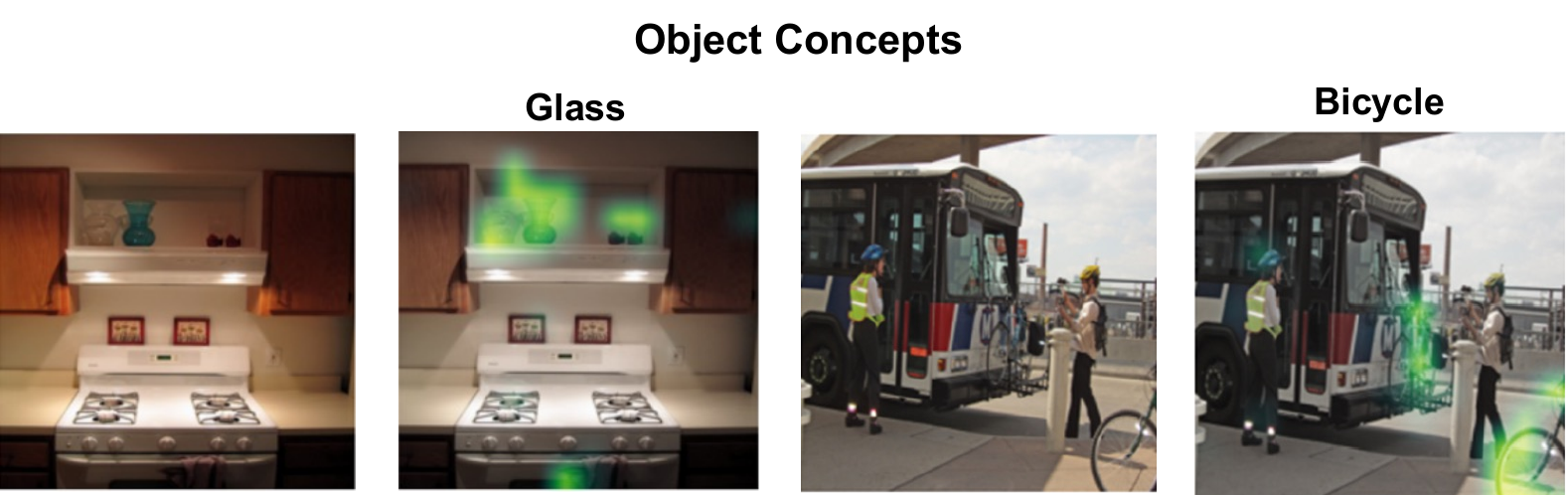}%
}\\[0.6em]
\makebox[\textwidth][c]{%
  \includegraphics[width=0.49\textwidth]{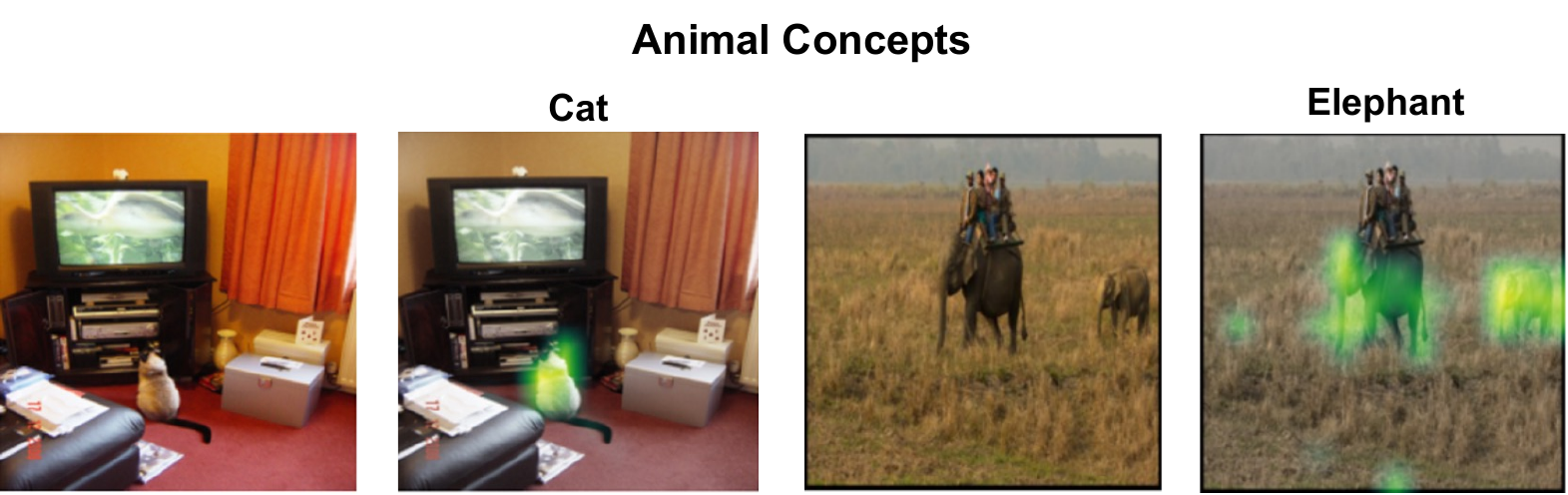}\hfill
  \includegraphics[width=0.49\textwidth]{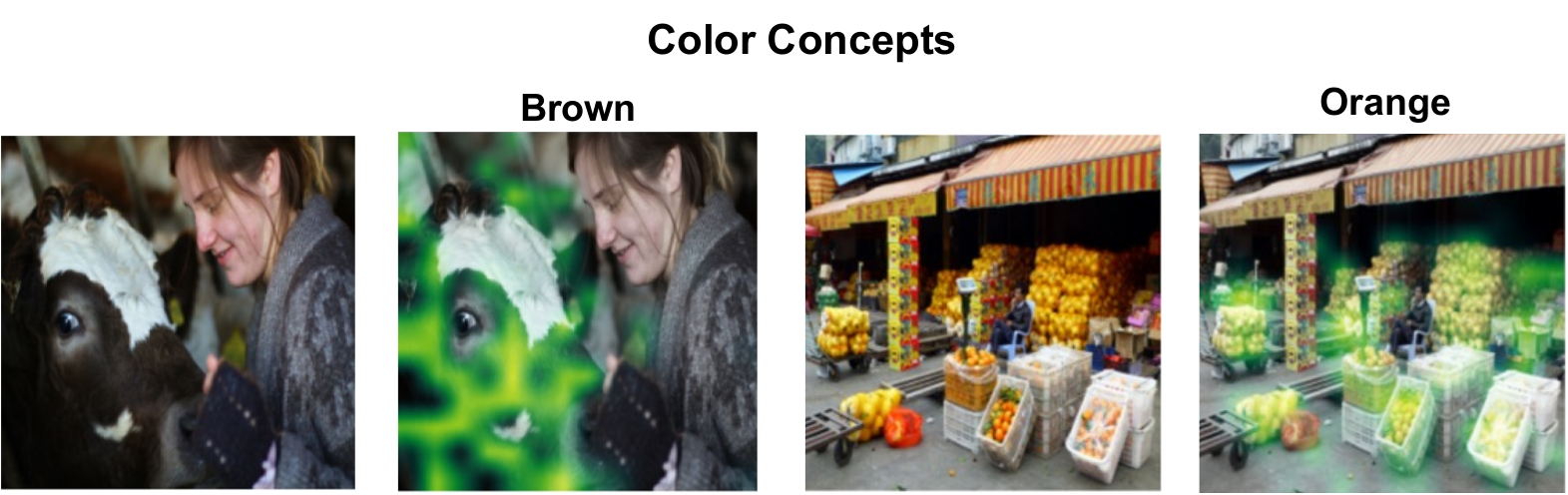}%
}\\[0.6em]
\makebox[\textwidth][c]{%
  \includegraphics[width=0.49\textwidth]{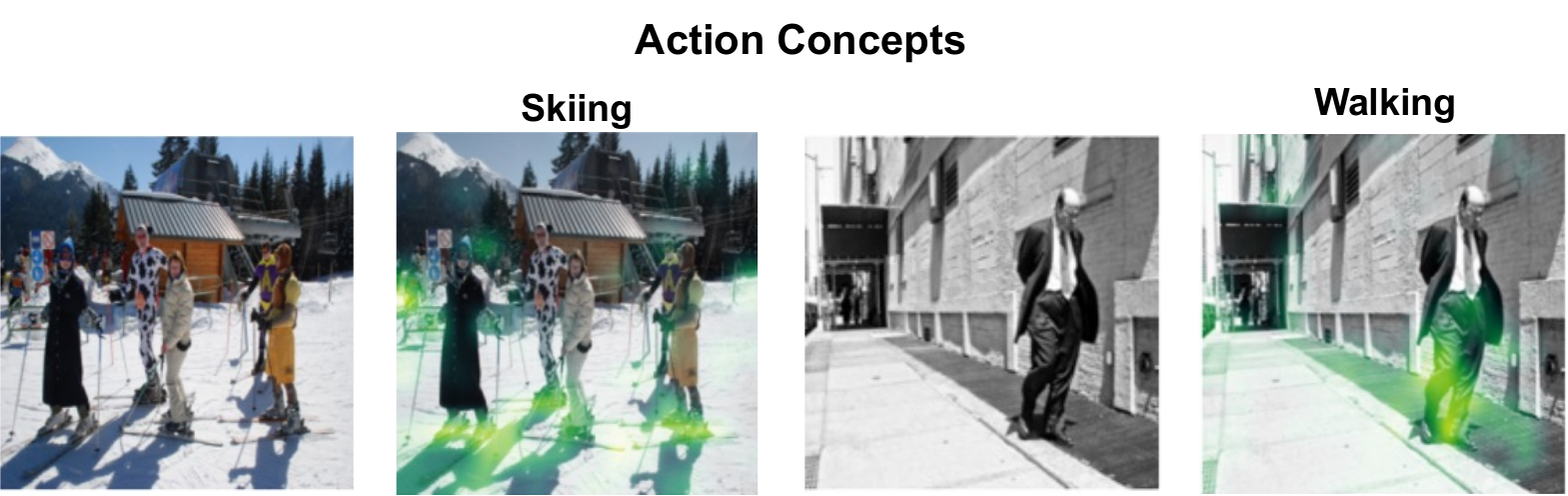}\hfill
  \includegraphics[width=0.49\textwidth]{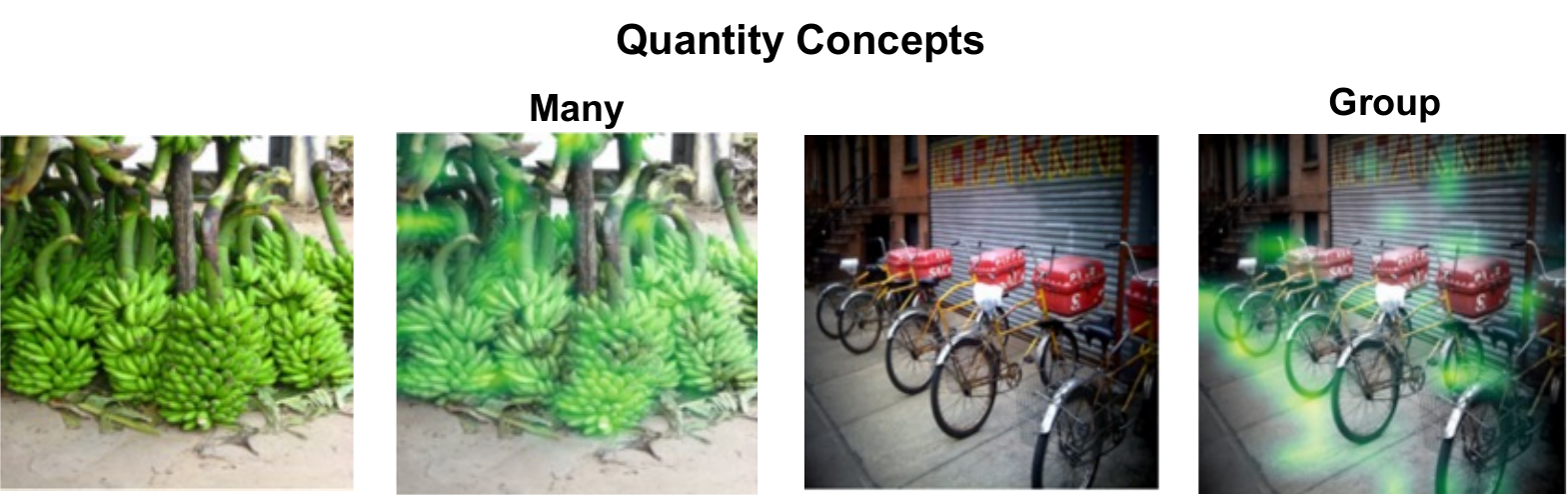}%
}
\caption{\textbf{Concept discovery across semantic categories.} Examples of diverse concepts discovered by LUCID on COCO images. For each concept, we show two images with activation heatmaps (green overlay) demonstrating spatial localization. LUCID captures concepts spanning scenes (\emph{Bathroom}, \emph{Bedroom}), objects (\emph{Glass}, \emph{Bicycle}), animals (\emph{Cat}, \emph{Elephant}), colors (\emph{Brown}, \emph{Orange}), actions (\emph{Skiing}, \emph{Walking}), and quantities (\emph{Many}, \emph{Group}), with activations precisely localizing to semantically relevant regions.}
\label{fig:coco_concepts}

\end{figure*}

\subsection{LUCID: Unified Sparse Coding with Shared and Private Capacity}

We aim to learn sparse codes that are both interpretable through sparsity and comparable across representation spaces. LUCID achieves this by decomposing the SAE dictionary space into a shared sparse code and a private sparse code. The shared code captures cross-space semantic concepts, while the private code accommodates modality-specific features.

\subsubsection{Architecture}

For each space $m \in \{\text{img}, \text{txt}\}$ and each element $h_n^{(m)} \in \mathbb{R}^{d_m}$, LUCID produces:
\begin{equation*}
(z_{s,n}^{(m)}, z_{p,n}^{(m)}) = E_m(h_n^{(m)}), \quad z_{s,n}^{(m)} \in \mathbb{R}^{K_s}, \quad z_{p,n}^{(m)} \in \mathbb{R}^{K_p},
\end{equation*}
where \(s\) and \(p\) means shared and private. \(E_m\) is a sparse encoder, and the output subject to Top-$k$ sparsity constraints~\cite{gao2024scaling} $\|z_{s,n}^{(m)}\|_0 \leq k_s$ and $\|z_{p,n}^{(m)}\|_0 \leq k_p$. Each space has a corresponding decoder $D_m$:
\begin{equation*}
\hat{h}_n^{(m)} = D_m(z_{s,n}^{(m)}, z_{p,n}^{(m)}).
\end{equation*}

We adopt a linear parameterization for the decoder:
\begin{equation*}
\hat{h}_t^{(m)} = W_s^{(m)} z_{s,n}^{(m)} + W_p^{(m)} z_{p,n}^{(m)} + b^{(m)},
\end{equation*}
where $W_s^{(m)} \in \mathbb{R}^{d_m \times K_s}$, and $W_p^{(m)} \in \mathbb{R}^{d_m \times K_p}$ captures shared and private structure specific to that space.

This design balances two competing objectives. Fully separate SAEs yield unaligned dictionaries where units are not comparable across spaces. Conversely, fully shared dictionaries can over-constrain reconstruction when spaces contain non-overlapping factors. The shared and private decomposition enables cross-space comparability while maintaining space-specific reconstruction fidelity.

\subsection{Cross-Space Alignment via Optimal Transport}

Because the number of elements $T_a$ and $T_b$ may differ (e.g., image patches versus text tokens), we employ Optimal Transport (OT) to define a principled many-to-many alignment between elements across paired samples $(H^{(a)}, H^{(b)})$.

\subsubsection{Transport Plan Formulation}

For each pair $n$, we define a cost matrix $C_n \in \mathbb{R}^{T_a \times T_b}$ whose entry $C_{n,ij}$ measures the mismatch between element $i$ in space $a$ and element $j$ in space $b$. We consider measuring distance in shared-code space with cosine similarities:
\begin{equation*}
C_{n,ij} = 1 - \cos(z_{s,n,i}^{(a)}, z_{s,n,j}^{(b)}).
\end{equation*}

Given marginal distributions $r_n \in \Delta^{T_a}$ and $c_n \in \Delta^{T_b}$ (strictly positive probability vectors), we compute an entropically regularized OT plan:
\begin{equation} \label{eq:ot}
\begin{aligned}
\Pi_n^\star = &\operatorname*{arg\,min}_{\Pi \in \mathbb{R}_+^{T_a \times T_b}} \langle C_n, \Pi \rangle - \varepsilon H(\Pi) \\
&\text{s.t.} \quad \Pi \mathbf{1} = r_n, \; \Pi^\top \mathbf{1} = c_n,
\end{aligned}
\end{equation}
where $H(\Pi) = -\sum_{ij} \Pi_{ij} \log \Pi_{ij}$ is the entropy and $\varepsilon > 0$ controls the softness of the assignment.

The OT framework offers several advantages over fixed correspondence assumptions. It naturally handles unequal set sizes, allows soft many-to-many correspondences, and provides a differentiable alignment signal. This flexibility is particularly well-suited for patch-token granularity differences, avoiding the restrictive assumption of fixed one-to-one pairings. Critically, the transport plan $\Pi_n^\star$ achieves alignment by minimizing the total cost $\langle C_n, \Pi_n^\star \rangle$, which directly encourages semantically corresponding elements across modalities to have similar representations in the shared-code space. The entropy regularization $-\varepsilon H(\Pi)$ prevents degenerate matchings and enables smooth, probabilistic correspondences that are robust to minor variations and redundancies in each modality. By incorporating the OT distance as a training loss, we provide a geometric regularization that pulls cross-modal representations closer together, thereby learning a shared-code space where alignment is structurally enforced rather than assumed. The following theorem establishes the closed-form solution to~\eqref{eq:ot} and utilizing sinkhorn algorithm~\cite{cuturi2013sinkhorn} shows linear time convergence for the problem, proof shown in~\ref{cor:sinkhorn}.

\begin{theorem}[Closed-Form Solution for Entropy-Regularized OT]
\label{thm:sinkhorn}
The unique optimal solution to~\eqref{eq:ot} admits the factorization:
\begin{equation*}
\Pi_n^\star = \mathrm{diag}(u) \, K \, \mathrm{diag}(v),
\end{equation*}
where $K_{ij} = \exp(-C_{n,ij}/\varepsilon)$ is the Gibbs kernel, and $u \in \mathbb{R}_{++}^{T_a}$, $v \in \mathbb{R}_{++}^{T_b}$ are scaling vectors uniquely determined (up to a multiplicative constant) by the marginal constraints.
\end{theorem}

\subsubsection{Guided Cross-Modal Transport (GCMT)}
While the standard OT formulation in~\eqref{eq:ot} provides flexible alignment, purely similarity-driven transport may align spurious or overly generic elements when pairings are weak or ambiguous. We extend our method to Guided Cross-Modal Transport (GCMT), which incorporates lightweight structural guidance into the OT formulation to bias transport toward semantically plausible correspondences. We denote the resulting guided transport plan as $\Pi_n^{\text{GCMT}}$.

\paragraph{Masked Transport (Structural Guidance).}
Let $M_n \in \{0,1\}^{T_a \times T_b}$ be a binary (or soft) mask indicating admissible alignments, such as phrase-to-region compatibility, noun-to-object priors, or geometric constraints. We incorporate $M_n$ by modifying the cost matrix:
\begin{equation*}
\tilde{C}_n = C_n + \beta(1 - M_n),
\end{equation*}
where large $\beta$ discourages forbidden pairs. The OT plan is then computed using $\tilde{C}_n$ in place of $C_n$ in~\eqref{eq:ot}.

\paragraph{Mass Reweighting (Grounding Guidance).}
GCMT biases transport mass by adjusting OT marginals based on global context compatibility. We compute a global visual context vector $g^{(a)} = \frac{1}{T_a} \sum_{i=1}^{T_a} z_{s,n,i}^{(a)}$ and measure each text token's alignment via $w_{n,j}^{(b)} = \max(0, \text{cos}(g^{(a)}, z_{s,n,j}^{(b)}))$. Text tokens inconsistent with the visual scene receive lower weights. The marginals are set as:
\begin{equation*}
r_n = \frac{1}{T_a} \mathbf{1}_{T_a}, \quad c_n = \frac{w_n^{(b)}}{\mathbf{1}^\top w_n^{(b)}}.
\end{equation*}
This compatibility can also be incorporated into the cost matrix: $\tilde{C}_n = C_n + \lambda_{\text{global}} \cdot (1 - w_{n}^{(b)})^\top \mathbf{1}_{T_a}$, where $\lambda_{\text{global}}$ controls the modulation strength.

\subsubsection{Alignment Loss}
\label{sec:alignment}

We encourage aligned elements to activate consistent shared codes. While the soft OT plan provides smooth gradient supervision, it can yield diffuse correspondences. In our implementation, we encourage crisper \emph{semantic} agreement by imposing a \emph{global semantic lock}. Concretely, for each sample $n$, we form pooled shared codes
$\bar z^{(a)}_{s,n}=\frac{1}{T_a}\sum_i z^{(a)}_{s,n,i}$ and
$\bar z^{(b)}_{s,n}=\frac{1}{|M_n|}\sum_{j\in M_n} z^{(b)}_{s,n,j}$, and pooled activations
$\bar h^{(a)}_{n}=\frac{1}{T_a}\sum_i h^{(a)}_{n,i}$ (and analogously for $b$). We then decode the pooled code from one modality through the other modality's decoder and regress to the pooled target:

\begin{equation*}
\mathcal{L}_{\text{align}} = \sum_{n}\bigl\|D^{(a)}(\bar z^{(b)}_{s,n})-\bar h^{(a)}_{n}\bigr\|_2^2,
\end{equation*}

with a symmetric formulation for $a \leftrightarrow b$. We treat the pooled targets $\bar h^{(\cdot)}$ as stop-gradient to maintain training stability while enforcing that shared dimensions carry consistent global semantics across the two modalities.

\subsection{Cross-Modal Reconstruction}
Beyond self-reconstruction, LUCID leverages the transport plan to perform cross-modal reconstruction, which validates that shared codes capture meaningful cross-modal semantics and provides supervision to strengthen concept alignment.

\subsubsection{Cross-Reconstruction via Barycentric Mapping}
Given the transport plan $\Pi_n^{\text{GCMT}}$, we map representations across modalities via barycentric projections. For vision-to-text and text-to-vision reconstruction:
\begin{equation*}
\begin{aligned}
\hat{h}_{n,i}^{(a \to b)} &=
\sum_j \frac{\Pi_{n,ij}}{\sum_{j'} \Pi_{n,ij'}} \cdot D_b(z_{s,n,j}^{(b)}), \\
\hat{h}_{n,j}^{(b \to a)} &=
\sum_i \frac{\Pi_{n,ij}}{\sum_{i'} \Pi_{n,i'j}} \cdot D_a(z_{s,n,i}^{(a)}).
\end{aligned}
\end{equation*}

where $D_a, D_b$ are the decoders. These reconstructions use only shared codes $z_s$, ensuring the transport plan operates on genuinely shared semantic content while modality-specific information remains in private codes $z_p$.

\subsubsection{Cross-Reconstruction Loss}
\label{sec:cross}
The cross-reconstruction loss measures reconstruction quality using transported shared codes:
\begin{equation*}
\mathcal{L}_{\text{cross}} = \sum_{n,i} \|h_{n,i}^{(a)} - \hat{h}_{n,i}^{(b \to a)}\|_2^2 + \sum_{n,j} \|h_{n,j}^{(b)} - \hat{h}_{n,j}^{(a \to b)}\|_2^2.
\end{equation*}
Combined with $\mathcal{L}_{\text{align}}$, this creates a bidirectional consistency constraint: shared codes must be both geometrically close in latent space and decodable to semantically equivalent representations in both modalities.

\subsection{Training Objective}
LUCID is trained end-to-end to reconstruct activations in each space while aligning shared codes across spaces. The full training objective is:
\begin{equation*}
\begin{aligned}
\mathcal{L} = \alpha \sum_n \sum_{m \in \{a,b\}} \sum_{t=1}^{T_m} \underbrace{\|h_{n,t}^{(m)} - \hat{h}_{n,t}^{(m)}\|_2^2}_{\mathcal{L}_{\text{self}}} \\ 
+ \beta \underbrace{\mathcal{L}_{\text{align}}}_{\text{OT alignment}} + \gamma \underbrace{\mathcal{L}_{\text{cross}}}_{\text{cross-recon}},
\end{aligned}
\end{equation*}

The self-reconstruction term $\mathcal{L}_{\text{self}}$ ensures faithful representation of the original activations within each modality. The alignment term $\mathcal{L}_{\text{align}}$ (from Section~\ref{sec:alignment}) enforces geometric proximity of corresponding shared codes using the OT-derived correspondences. The cross-reconstruction term $\mathcal{L}_{\text{cross}}$ (from Section~\ref{sec:cross}) enforces that shared codes are functionally equivalent across spaces when decoded back to their original representation spaces. The sparsity terms enforce interpretability through sparse activation patterns, with separate regularization strengths for shared ($\lambda_s$) and private ($\lambda_p$) codes. The hyperparameters $\beta$ and $\gamma$ control the relative importance of these objectives. Detailed parameter settings are provided in appendix~\ref{app:implementation}.

\begin{table*}[t]
\centering
\small
\setlength{\tabcolsep}{6pt}
\renewcommand{\arraystretch}{1.15}
\begin{tabular}{p{2.2cm}p{3.2cm}p{10.2cm}}
\hline
\textbf{Neuron ID} & \textbf{Semantic} & \textbf{Top terms (dominant semantic cluster)} \\
\hline
178 & Urban Street Scene & traffic light, traffic cone, pavement, street sign, street vendor, street, windy, driver, sidewalk, stop sign, sandwich, pedestrian, pedestrian signal, safety vest, drawer, wheel, magazine, scooter \\
54 & Aquatic Activities & water reflection, water surface, water, lake, apple, swimming, wings, river, surfboard, boat, boat wake, cyclist, ocean, wave, knife, harness, dog, dog walker \\
\hline
\end{tabular}
\caption{Semantic interpretations of SAE neurons obtained via clustering their top-activating text concepts. We show two example neurons: Neuron 178 captures urban street features, while Neuron 54 captures aquatic/water-related features. An extended table with additional neurons is provided in Appendix~\ref{app:additional_viz}.}
\label{tab:neuron_terms_short}
\end{table*}

\begin{figure*}
\centering

\includegraphics[width=0.49\linewidth]{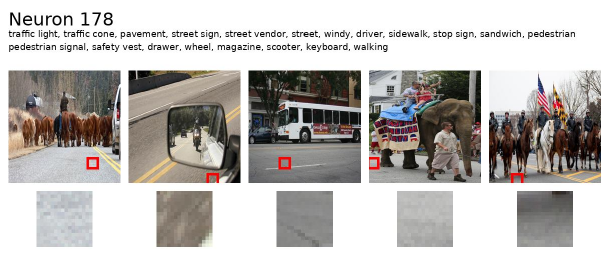}\hfill
\includegraphics[width=0.49\linewidth]{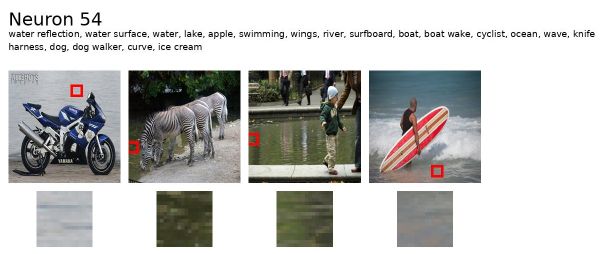}\\[0.6em]
\caption{Visualization of two randomly selected neurons. For each neuron, we show the highest-scoring retrieved images and highlight the most activated patch (red box), along with the corresponding patch crops. Additional neuron visualizations are provided in Appendixin Appendix~\ref{app:additional_viz}.}
\label{fig:all22_short}
\vspace{-1em}
\end{figure*}

\subsection{Automated Neuron Interpretation via Concept Clustering}
\label{sec:neuron_interp}

A key advantage of aligned multimodal dictionaries is the ability to automatically interpret shared neurons by leveraging cross-modal correspondences. We employ a pipeline that clusters text concepts associated with a neuron’s activations, revealing semantic themes without manual annotation.

\subsubsection{Interpretation Pipeline}
The interpretation process is summarized into three phases (see appendix for algorithm~\ref{alg:neuron_interp_simple}):

\begin{itemize}
    \item \textbf{Stage 1: Concept Encoding \& Weighting.} We encode a concept library $\mathcal{C}$ into shared-code representations $z_c$. To down-weight generic features, we apply an inverse document frequency (IDF) weight, $\text{idf}_m = \log(N / \text{df}_m)$, based on neuron activation frequency across the library.
    \item \textbf{Stage 2: Spatial Concept Matching.} We compute spatial activation maps $\mathcal{A}_i(p) = \min(z_{p}^{(a)}, z_{c_i}^{(b)}) \odot \text{idf}$ to identify the intersection between image patches and concepts. Concepts are associated with the dominant neuron $m^*$ at their peak activation location.
    \item \textbf{Stage 3: Semantic Clustering.} Associated concepts are clustered using sentence-level embeddings via greedy cosine similarity (threshold $\sigma \approx 0.78$). The largest cluster defines the neuron's primary semantic specialization.
\end{itemize}

\subsubsection{Qualitative Results}
We filter for interpretability by requiring a minimum number of concept hits and a dominant cluster purity $\geq 0.55$. Table~\ref{tab:neuron_terms_short} and Figure~\ref{fig:all22_short} illustrate this mapping for two representative neurons. \textbf{Neuron 178} specializes in urban navigation, grouping terms like \textit{traffic light}, \textit{sidewalk}, and \textit{stop sign}. In contrast, \textbf{Neuron 54} captures aquatic themes, clustering \textit{lake}, \textit{river}, and \textit{surfboard}. 

The high cluster purity validates that shared neurons capture consistent, cross-modal concepts. This automated pipeline scales to thousands of neurons, providing a systematic lens into the model's learned representations.


\begin{figure*}[t]
\centering
\makebox[\textwidth][c]{%
  \includegraphics[width=0.49\textwidth]{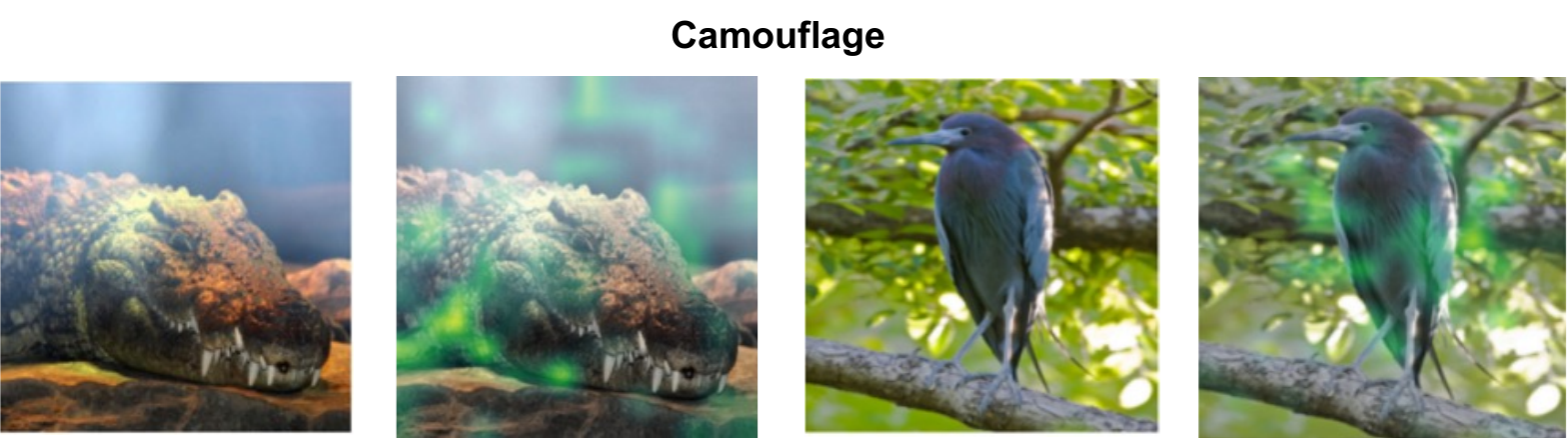}\hfill
  \includegraphics[width=0.49\textwidth]{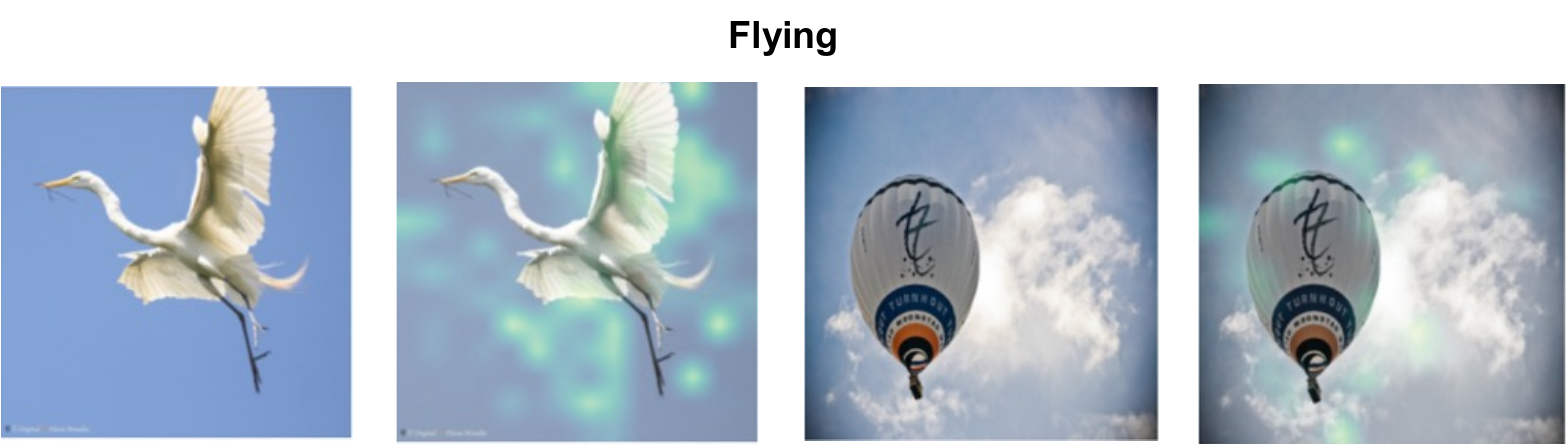}%
}\\[0.6em]
\makebox[\textwidth][c]{%
  \includegraphics[width=0.49\textwidth]{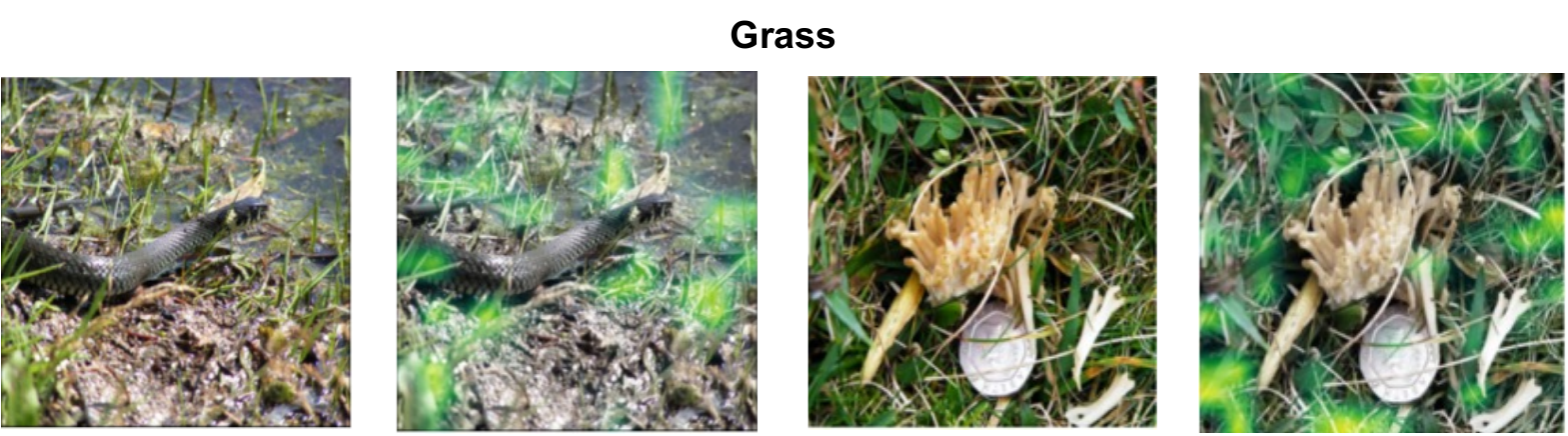}\hfill
  \includegraphics[width=0.49\textwidth]{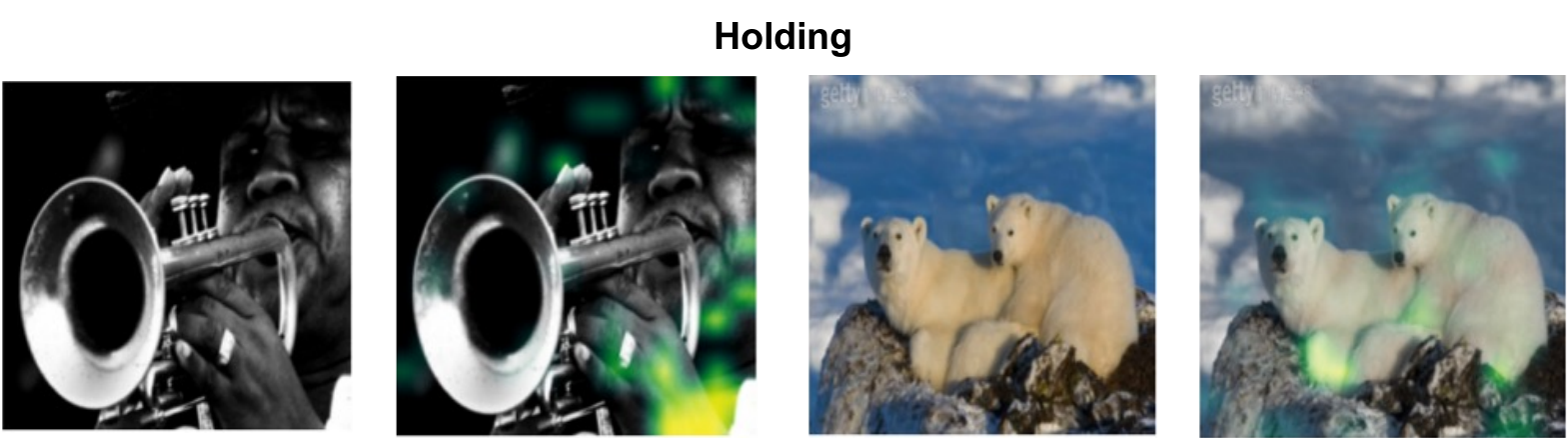}%
}\\[0.6em]
\makebox[\textwidth][c]{%
  \includegraphics[width=0.49\textwidth]{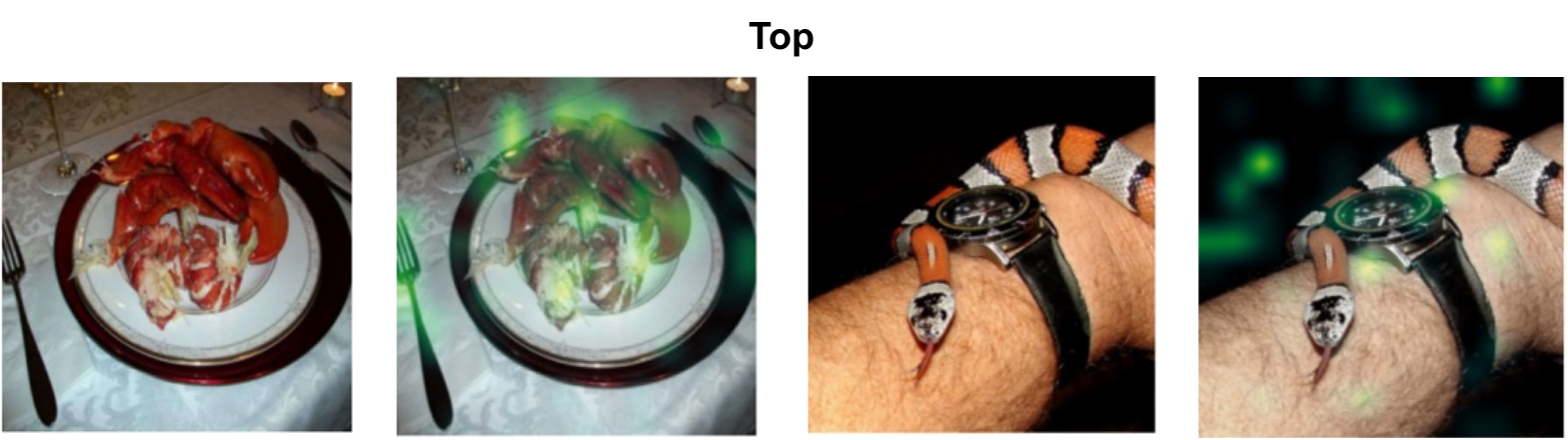}\hfill
  \includegraphics[width=0.49\textwidth]{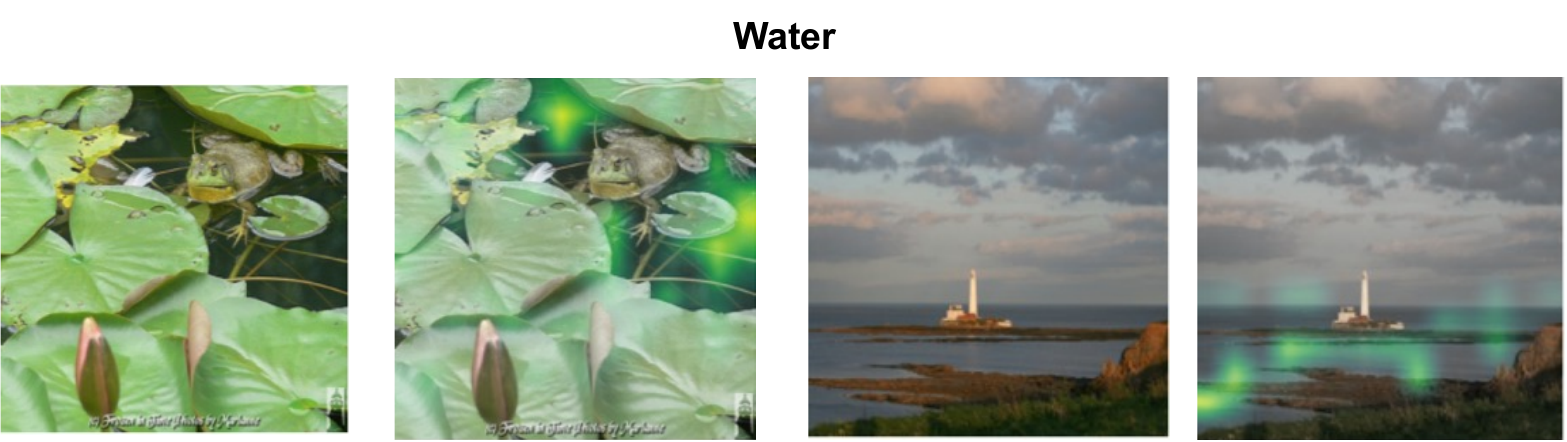}%
}
\caption{Out-of-distribution data (ImageNet) used to validate concept discovery with LUCID. For each discovered concept (e.g., \emph{camouflage}, \emph{flying}, \emph{grass}, \emph{holding}, \emph{top}, \emph{water}), we show two representative images along with activation overlays highlighting the spatial evidence supporting the concept.}
\label{fig:imagenet_concepts}
\end{figure*}

\section{Experimental Results}
This section validates three core contributions of our method: (1) the shared-private capacity decomposition improves both reconstruction fidelity and cross-modal predictability of the shared subspace, (2) optimal transport coupling enables precise object-level grounding measured through local alignment metrics, and (3) structured shared-code regularization mitigates clustering degeneracy while improving grounding.
\subsection{Experimental setup}
We train our approach on combined MS COCO 2017~\cite{lin2014microsoft} and CC3M~\cite{sharma2018conceptual} with their corresponding image-text caption pair. No ground truths are used during training. When evaluating, we use the COCO instance annotations to obtain ground-truth object regions (bounding boxes; masks are a drop-in replacement when available). Our backbone is SigLIP ViT-B/16~\cite{zhai2023sigmoid,dosovitskiy2020image}, and our model learns sparse concept codes over image patch tokens and text tokens.

\begin{table}[t]
\centering
\scriptsize
\setlength{\tabcolsep}{2pt}
\renewcommand{\arraystretch}{0.95}
\resizebox{\columnwidth}{!}{%
\begin{tabular}{lcc}
\toprule
\textbf{Metric} & \textbf{Shared Only} ($r{=}1.0$) & \textbf{Shared+Priv} ($r{=}0.25$) \\
\midrule
$R^2_{\mathrm{V,self}}$& 0.4045 & \textbf{0.4606} \\
$R^2_{\mathrm{V,self}}\scriptstyle(\mathrm{shared})$       & - & 0.2705 \\
$R^2_{\mathrm{V,self}}\scriptstyle(\mathrm{private})$      & - & 0.2445 \\
\addlinespace[2pt]
$R^2_{\mathrm{T,self}}$& 0.5737 & \textbf{0.6684} \\
$R^2_{\mathrm{T,self}}\scriptstyle(\mathrm{shared})$       & - & 0.5088 \\
$R^2_{\mathrm{T,self}}\scriptstyle(\mathrm{private})$      & - & 0.2550 \\
\midrule
$R^2_{\mathrm{T\rightarrow V,cross}}\scriptstyle(\mathrm{shared})$ & 0.0552 & \textbf{0.1171} \\
$R^2_{\mathrm{V\rightarrow T,cross}}\scriptstyle(\mathrm{shared})$ & 0.3220 & \textbf{0.4055} \\
\bottomrule
\end{tabular}}
\caption{\textbf{Shared–private decomposition improves reconstruction.}
Adding private capacity ($r{=}0.25$) increases self $R^2$ while improving shared cross $R^2$ ($T{\to}V$, $V{\to}T$). For $r{=}0.25$ we also report shared vs.\ private contributions to self $R^2$.}
\label{tab:r2_main_condensed_direct}
\vspace{-4em}
\end{table}

\subsection{Concept Discovery and Interpretability}
\label{sec:concept_discovery}

A key advantage of our learned aligned sparse dictionaries is the ability to discover and visualize semantically meaningful concepts directly from the learned representations. Figure~\ref{fig:coco_concepts} showcases diverse concept types discovered by LUCID on COCO across randomly selected semantic categories, ie. scene concepts (bathroom, bedroom) and action concepts (skiing, walking). The heat map images reveal precise spatial localization, ie.``glass" activates on drinking glasses, ``bicycle" on bike frames. Scene concepts like ``bedroom" activate on characteristic furniture arrangements, action concepts like ``skiing" respond to human poses and contextual cues, and quantity concepts like ``many" capture compositional properties. These results demonstrate that the shared space captures both low-level perceptual attributes and high-level semantic categories. Since these activation maps measure intersection between vision patch codes and text concept codes in the shared space, the spatial precision validates that our transport-based alignment successfully enforces semantic correspondence at the neuron level. Figure~\ref{fig:imagenet_concepts} validates that discovered concepts generalize beyond the training distribution by evaluating the same learned dictionary on ImageNet~\cite{russakovsky2014imagenet}. We select random category concepts: camouflage, flying, grass, holding, top, and water. The activation patterns remain semantically meaningful and spatially precise. ``Camouflage" activates on natural patterns, ``flying" responds to birds with sky context, ``grass" localizes to ground vegetation, and ``holding" identifies object-hand interactions. This generalization demonstrates that LUCID learns transferable semantic concepts rather than dataset-specific correlations, supporting the use of learned representations for downstream zero-shot tasks. Additional concept discovery visualizations are provided in Appendix~\ref{app:additional_viz}.

\subsection{Shared-Private Capacity Decomposition}

Table~\ref{tab:r2_main_condensed_direct} validates that private capacity improves both reconstruction fidelity and cross-modal predictability. Introducing private capacity ($r{=}0.25$) improves overall self-reconstruction ($R^2_{\text{V,self}}$: $0.40 \rightarrow 0.46$; $R^2_{\text{T,self}}$: $0.57 \rightarrow 0.67$) compared to shared-only ($r{=}1.0$). The decomposition reveals that in vision, private-only contribution ($0.2445$) nearly equals the shared-only path ($0.2705$), confirming that substantial visual variance is modality-specific. In text, the shared-only component remains dominant ($0.5088$), while private units capture residual information.

Crucially, private capacity \emph{strengthens} the shared subspace as a cross-modal interface. Shared-only cross-predictability improves significantly: $R^2_{\mathrm{T\rightarrow V,cross}}$ increases from $0.055$ to $0.117$, and $R^2_{\mathrm{V\rightarrow T,cross}}$ from $0.322$ to $0.406$. This confirms that allocating private capacity to absorb modality-specific details allows the shared dictionary to specialize in genuinely comparable features, shielding it from modality-dependent noise. We note a persistent asymmetry where text$\rightarrow$vision performance is lower; this is expected, as semantically compressed captions cannot fully recover dense visual details like background and layout.


\begin{table}[t]
\centering
\scriptsize
\setlength{\tabcolsep}{10pt}
\renewcommand{\arraystretch}{1.05}
\begin{tabular}{lccc}
\toprule
\textbf{Configuration} & \textbf{mass@obj} & \textbf{point@1} & \textbf{IoU@10} \\
\midrule
OT only           & 0.4020 & 0.4730 & 0.2746 \\
Share k           & 0.4039 & 0.5501 & 0.2800 \\
Share k GSL      & 0.4097 & 0.5578 & 0.2913 \\
Share k GSL GCMT & \textbf{0.4106} & \textbf{0.6002} & \textbf{0.2984} \\
\bottomrule
\end{tabular}
\caption{\textbf{Object-level grounding via BBOX metrics (higher is better).}
We form a shared-code spatial heatmap and compare it to ground-truth (GT) boxes. Metric definition in appendix~\ref{app:implementation}. Results compare OT-only and shared Top-$k$ variants, with the best score in each column in bold.}
\label{tab:claimb_bbox_grounding}
\vspace{-2em}
\end{table}

\subsection{Object-Level Grounding via Optimal Transport}

Unlike global instance discrimination, our approach optimizes \emph{token-to-patch} alignment using Optimal Transport (OT). Because global retrieval metrics can be obscured by artifacts like concept clustering, we evaluate local grounding directly using OT-based heatmaps. By aggregating transport mass $P_{ij}$ over tokens for each patch $j$, we generate a spatial grid compared against ground-truth (GT) bounding boxes via three metrics: \textbf{mass@obj} (mass fraction within GT), \textbf{point@1} (peak saliency accuracy), and \textbf{IoU@10} (overlap of the top 10\% salient patches).

Table~\ref{tab:claimb_bbox_grounding} demonstrates that enforcing sparse, structured shared codes significantly sharpens grounding. Compared to the OT baseline, the Share k variant improves point@1 from $0.473$ to $0.550$, indicating that sparsification helps the model localize salient features. Adding a \emph{Global Semantic Lock (GSL)} further improves all metrics (mass@obj: $0.410$, point@1: $0.558$, IoU@10: $0.291$), consistent with better cross-modal shared-code semantics. Incorporating GCMT on top of GSL yields the best overall performance (mass@obj: $0.411$, point@1: $0.600$, IoU@10: $0.298$). These gains confirm that Top-$k$ selection reduces diffuse activations while GCMT encourages context-aware, object-focused couplings. Together, shared-code sparsity and GSL/GCMT sharpen transport concentration on semantically relevant object regions, providing a more faithful measure of local alignment than global retrieval alone.


\begin{table}[t]
\centering
\scriptsize
\setlength{\tabcolsep}{3pt}
\renewcommand{\arraystretch}{1.0}
\begin{tabular}{lccc}
\toprule
\textbf{Configuration} & \textbf{Clust t2i maxfreq $\downarrow$} & \textbf{Clust i2t maxfreq $\downarrow$} & \textbf{BBOX point@1 $\uparrow$} \\
\midrule
OT only            & 0.7958 & 0.3810 & 0.4730 \\
Share k            & 0.6516 & 0.2055 & 0.5501 \\
Share k GSL       & 0.4590 & 0.1240 & 0.5578 \\
Share k GSL GCMT  & \textbf{0.3538} & \textbf{0.0740} & \textbf{0.6002} \\
\bottomrule
\end{tabular}
\caption{\textbf{Degeneracy vs. grounding.}
Share-$k$ + GSL/GCMT reduce clustering maxfreq (t2i/i2t) while improving BBOX point@1.}

\label{tab:claimC_clustering_tradeoff}
\vspace{-2em}
\end{table}

\subsection{Sparsity and the Selectivity--Degeneracy Trade-off}
Our third contribution shows that structured shared codes provide a practical knob for balancing \emph{selectivity} against \emph{degeneracy} in cross-modal concept representations. Table~\ref{tab:claimC_clustering_tradeoff} characterizes this interaction using two complementary signals: a concept-clustering diagnostic (\textbf{maxfreq}) that captures dimension overuse, and an object-level localization metric (\textbf{BBOX point@1}) that evaluates spatial grounding accuracy. We compute maxfreq by measuring, for each shared dimension $m$, its activation frequency over the evaluation set $\mathcal{D}$ and taking the maximum:
\begin{equation}
\mathrm{maxfreq} \;=\; \max_{m}\;\frac{1}{|\mathcal{D}|}\sum_{x\in\mathcal{D}} \mathbb{1}\!\left[z_m(x)>0\right],
\end{equation}
where $z(x)\in\mathbb{R}^{d_{\mathrm{sh}}}_{\ge 0}$ is the shared sparse code. The BBOX point@1 metric tests whether the peak of the shared-code heatmap falls inside a ground-truth bounding box, providing a direct measure of whether the learned correspondences localize salient object regions.

The table indicates that degeneracy is strongly affected by how shared codes are \emph{regularized and coupled} across modalities. Relative to OT only, Share $k$ reduces maxfreq (t2i: $0.7958\!\rightarrow\!0.6516$, i2t: $0.3810\!\rightarrow\!0.2055$), suggesting fewer dimensions dominate across clustered concept queries. Adding \emph{Global Semantic Lock (GSL)} further suppresses concentration (t2i: $0.4590$, i2t: $0.1240$) while slightly improving grounding (point@1: $0.5578$), consistent with better semantic consistency of the shared basis. Finally, incorporating GCMT yields the lowest maxfreq (t2i: $0.3538$, i2t: $0.0740$) and the best point@1 ($0.6002$), indicating that context-aware transport reduces spurious token matches that otherwise promote repeated cluster assignments.

\section{Conclusion}

We introduced LUCID, a unified sparse autoencoder framework that learns interpretable, cross-modal concept dictionaries by decomposing the latent space into shared and private components. This design addresses a fundamental limitation of existing sparse coding approaches: the inability to align concepts across modalities while preserving reconstruction fidelity. Our empirical validation demonstrates three key contributions: (1) shared-private decomposition improves both reconstruction quality and cross-modal predictability by shielding the shared space from modality-specific noise, (2) optimal transport-based alignment enables precise patch-level grounding as measured by local spatial metrics, and (3) structured shared-code objectives mitigate clustering degeneracy while sharpening object-level grounding.

Leveraging these alignment properties, we developed an automated neuron interpretation pipeline that discovers coherent semantic clusters without manual annotation. LUCID captures diverse concept categories spanning objects, scenes, actions, attributes, and spatial relationships, with learned representations generalizing to out-of-distribution data. This establishes aligned sparse dictionaries as a foundation for automatic interpretability at scale, providing a principled approach toward interpretable multimodal representations that support systematic cross-modal analysis.





\nocite{langley00}

\bibliography{example_paper}
\bibliographystyle{icml2026}

\newpage
\appendix
\onecolumn
\section{Appendix}



\subsection{Additional Implementation Details}
\label{app:implementation}

\paragraph{Model Architecture.}
We use SigLIP ViT-B/16 as our vision-language backbone, extracting patch-level features from the vision encoder and token-level features from the text encoder. For vision, we use the penultimate layer activations (before the final projection), yielding $196$ patch tokens of dimension $768$. For text, we extract token embeddings after the transformer layers but before pooling, with dimension $768$. Both SAEs map these $768$-dimensional activations to a shared dictionary of size $M_{\text{shared}} = 1536$ (25\% of total capacity) and private dictionaries of size $M_{\text{private}} = 4608$ (75\% of capacity). We use Top-$k$ sparsity with $k_{\text{shared}} = 16$ and $k_{\text{private}} = 32$.

\paragraph{Training Details.}
We train on MS COCO 2017 train split (118k image-caption pairs) for 200 epochs with batch size 256. The optimizer is AdamW with learning rate $3 \times 10^{-4}$, weight decay $0.01$, and cosine annealing schedule. Loss weights are set to $\alpha = 1.0$ (self-reconstruction), $\beta = 0.5$ (alignment), and $\gamma = 0.3$ (cross-reconstruction). For optimal transport, we use entropy regularization $\varepsilon = 0.1$ and run 10 Sinkhorn iterations. The GCMT guidance uses $\lambda_{\text{global}} = 0.5$ for global context modulation. Training takes approximately 24 hours on 8 NVIDIA RTX 8000 GPUs.

\paragraph{Evaluation Metrics.}
\textbf{Self-reconstruction $R^2$} measures variance explained by the SAE reconstruction: $R^2 = 1 - \frac{\sum_i \|h_i - \hat{h}_i\|^2}{\sum_i \|h_i - \bar{h}\|^2}$, where $\bar{h}$ is the mean activation.
\textbf{Cross-reconstruction $R^2$} measures how well codes from one modality predict the other under the learned transport plan.
\textbf{BBOX metrics} evaluate spatial grounding: mass@obj is the fraction of heatmap mass inside ground-truth boxes, point@1 is pointing-game accuracy (whether peak activation lands in a box), and IoU@10 measures overlap between boxes and the top-10\% salient region.
\textbf{Clustering maxfreq} measures degeneracy by finding the maximum activation frequency of any shared dimension across the evaluation set.

\paragraph{Extended Reconstruction Analysis.}
Table~\ref{tab:r2_appendix_extensive} provides a comprehensive breakdown of reconstruction quality under different coding strategies. The key insight is that the shared+private decomposition improves both self-reconstruction and cross-predictability compared to forced sharing. For forced sharing ($r=1.0$), all capacity is shared, so full global = full joint = shared. With shared+private ($r=0.25$), we observe: (1) Full joint reconstruction (shared + private) outperforms shared-only, demonstrating that private capacity captures additional variance. (2) Shared-only cross-reconstruction improves with private capacity, indicating that private codes absorb modality-specific noise that would otherwise contaminate the shared space.

\subsection{Concept Interpretation Pipeline Details}
\label{app:interp_details}

Our automated interpretation pipeline (Algorithm~\ref{alg:neuron_interp_simple}) operates in four stages. We provide additional implementation details for each stage.

\begin{algorithm}[t]
\caption{Automated Neuron Interpretation Pipeline}
\label{alg:neuron_interp_simple}
\begin{algorithmic}[1]
\REQUIRE Concept list $\mathcal{C}$, images $\mathcal{I}$, encoders and SAEs, thresholds
\ENSURE Semantic clusters for each shared neuron

\STATE \textbf{Stage 1: Encode Concepts}
\FOR{each concept $c$ in $\mathcal{C}$}
  \STATE Encode $c$ through text pathway to get shared code $z_c$
\ENDFOR

\STATE \textbf{Stage 2: Compute IDF Weights (optional)}
\FOR{each neuron $m$}
  \STATE Count how many concepts activate neuron $m$
  \STATE Compute inverse document frequency: $\text{idf}_m = \log(\text{\# concepts} / \text{\# activating } m)$
\ENDFOR

\STATE \textbf{Stage 3: Match Images to Concepts}
\STATE Initialize empty concept lists for each neuron
\FOR{each image $I$ in $\mathcal{I}$}
  \STATE Extract patch-level shared codes from vision pathway
  \FOR{each concept $c$}
    \STATE Compute activation map: intersection of image patches with concept code
    \STATE Find peak activation location and dominant neuron at that location
    \IF{peak confidence $>$ threshold}
      \STATE Associate concept $c$ with the dominant neuron
    \ENDIF
  \ENDFOR
\ENDFOR

\STATE \textbf{Stage 4: Cluster Concepts per Neuron}
\FOR{each neuron with associated concepts}
  \STATE Encode all associated concepts using VLM text encoder (for clustering)
  \STATE Group concepts by semantic similarity (greedy cosine clustering)
  \STATE Rank clusters by size
\ENDFOR

Clustered concepts for each neuron
\end{algorithmic}
\end{algorithm}

\begin{table*}[t]
\centering
\small
\setlength{\tabcolsep}{6pt}
\renewcommand{\arraystretch}{1.15}
\begin{tabular}{p{2.2cm}p{3.2cm}p{10.2cm}}
\hline
\textbf{Neuron ID} & \textbf{Semantic} & \textbf{Top terms (dominant semantic cluster)} \\
\hline
956 & Kitchen \& Indoor & cutting, cutting board, kitchen, kitchen sink, cabinet, countertop, stove, wallet, cookie, stone, living room, shoe, shelf, recliner, oven, knife, door, attached to \\
178 & Urban Street Scene & traffic light, traffic cone, pavement, street sign, street vendor, street, windy, driver, sidewalk, stop sign, sandwich, pedestrian, pedestrian signal, safety vest, drawer, wheel, magazine, scooter \\
1293 & Spatial Relationships & in hand, in mouth, in basket, in sink, in front of, mirrored, on table, on floor, on wall, on shelf, on head, on plate, throwing, horns, surfboard, flying, foot, skirt \\
1292 & Office / Desk Scene & camera, writing, paper, laptop, tape measure, next to, bag, office chair, desk, keyboard, notebook, box, newspaper, book, cup, foot, camouflage, backlit \\
1485 & Overcast Coastal / Outdoor Scene & overcast, over, smoothie, microwave, tail, taillight, depth layers, shallow depth, beach shore, beach, horizon, oval, cloudy, market day, market, beer, bleachers, single \\
1403 & Pasture / Roadside Outdoors & grass field, grass, ground, drill, grazing, road barrier, road, hand, bush, green, pantry, sheep, right, metallic, leaves, oval, black, broccoli \\
1515 & Dining / Bar Setting & spoon, wine, wine glass, pants, barstool, juice, basket storage, basket, tea, kettle, plate, menu, soda, pie, cocktail, cake, wedding, dining table \\
221 & Rail / Transit Corridor & waterfall, rail track, tunnel, platform, can, center, cooking, distant, palm tree, train, train window, station interior, rubber, bicycle, paintbrush, running, speaker, motion blur \\
1100 & Human Poses / Activities & eating, holding, sitting, twilight, tie, life jacket, sleeping, rectangle, sofa, seesaw, camouflage, leaves, lying, mouth, remote, red, wooden, cup \\
834 & Farm Animals & horse, spots, behind, cow, harness, faucet, hair, mouth, cactus, drinking, sheep, mouse, brown, crowd, motion blur, speaker, cutting, cutting board \\
1217 & Beach/Ocean + Crowd & large, medium, van, wave, ocean, palm tree, ice cream, horizon, picnic, concert, passenger, station interior, cupcake, line, tablet, star, ticket \\
54 & Aquatic Activities & water reflection, water surface, water, lake, apple, swimming, wings, river, surfboard, boat, boat wake, cyclist, ocean, wave, knife, harness, dog, dog walker \\
\hline
\end{tabular}
\caption{Semantic interpretations of SAE neurons obtained via clustering their top-activating text concepts (Extended version).}
\label{tab:neuron_terms}
\end{table*}

\begin{figure}[p]
\centering
\label{fig:all22}

\includegraphics[width=0.49\linewidth]{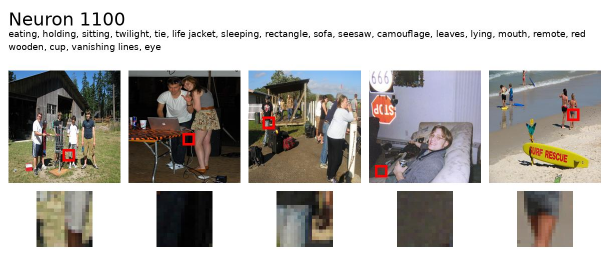}\hfill
\includegraphics[width=0.49\linewidth]{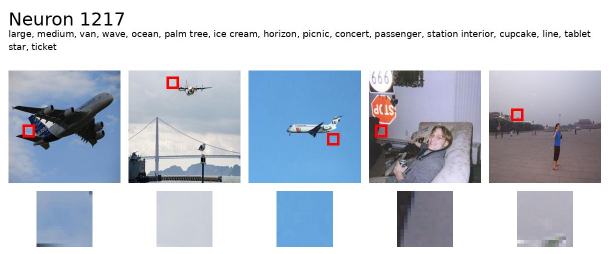}\\[0.6em]
\includegraphics[width=0.49\linewidth]{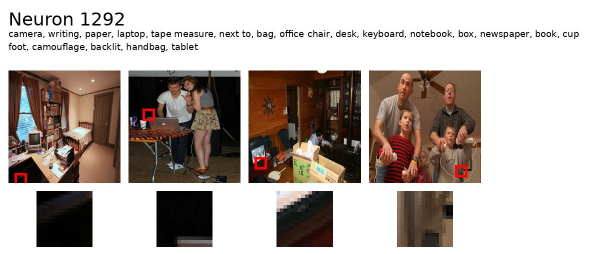}\hfill
\includegraphics[width=0.49\linewidth]{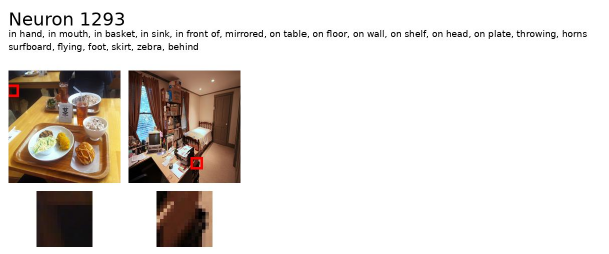}\\[0.6em]
\includegraphics[width=0.49\linewidth]{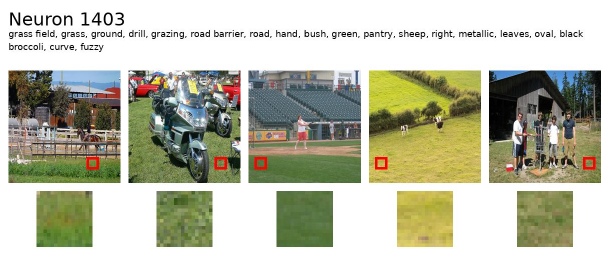}\hfill
\includegraphics[width=0.49\linewidth]{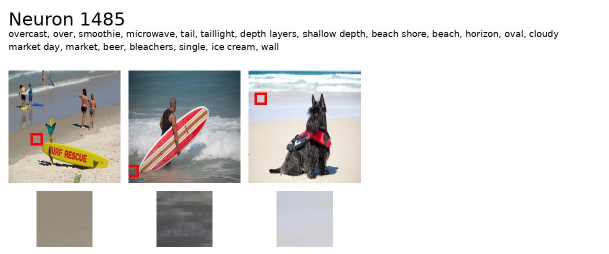}\\[0.6em]
\includegraphics[width=0.49\linewidth]{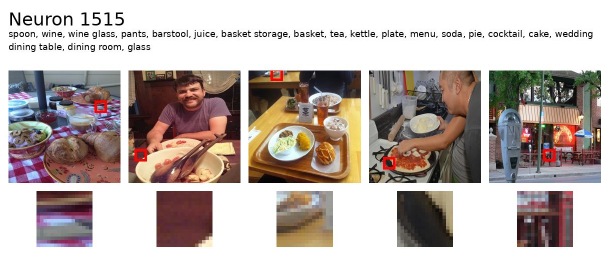}\hfill
\includegraphics[width=0.49\linewidth]{images/neuron_178.pdf}\\[0.6em]
\includegraphics[width=0.49\linewidth]{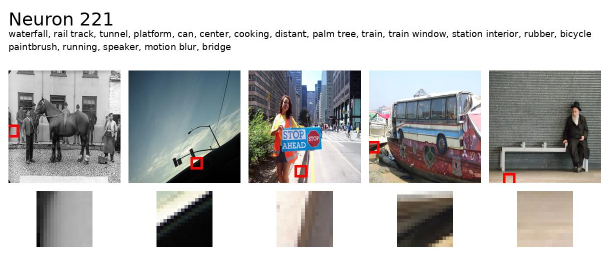}\hfill
\includegraphics[width=0.49\linewidth]{images/neuron_54.pdf}\\[0.6em]
\includegraphics[width=0.49\linewidth]{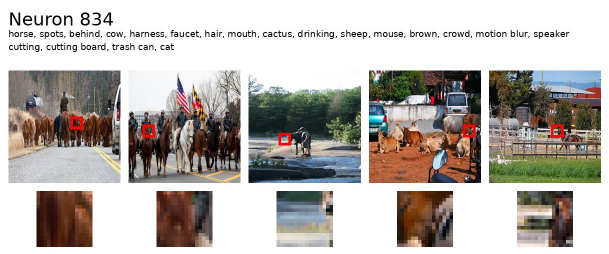}\hfill
\includegraphics[width=0.49\linewidth]{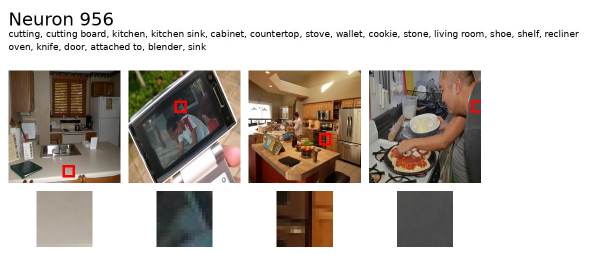}\\[0.6em]
\caption{Visualization of two randomly selected neurons (Extended version).}
\label{fig:neuron_visualize}
\end{figure}

\subsection{Additional Visualizations}
\label{app:additional_viz}

Table~\ref{tab:neuron_terms} shows the complete set of 12 randomly selected neurons with their semantic interpretations. Each neuron exhibits a dominant cluster of related concepts, demonstrating the diversity of learned representations. Figures~\ref{fig:all22} and~\ref{fig:neuron_visualize} provide extensive visualizations of concept activations and neuron-specific image examples. For each neuron, we show the top-scoring images (by peak activation) and highlight the most salient patch. These visualizations reveal that: (1) Neurons specialize for coherent semantic categories (e.g., Neuron 1515 consistently activates on dining scenes with tableware). (2) Spatial activations are precise: highlighted patches correspond to relevant objects or regions rather than backgrounds. (3) Neurons generalize across visual variations: Neuron 54 (aquatic activities) activates on diverse water-related scenes including lakes, rivers, boats, and swimming.

\begin{table*}[t]
\centering
\footnotesize
\setlength{\tabcolsep}{6pt}
\begin{tabular}{lcc}
\toprule
\textbf{Metric} & \textbf{Shared Only} ($\texttt{shared\_ratio}=1.0$) & \textbf{Shared+Private} ($\texttt{shared\_ratio}=0.25$) \\
\midrule
$R^2_{\mathrm{V,self}}(\textbf{full global})$ & 0.4045 & - \\
$R^2_{\mathrm{V,self}}(\textbf{full joint})$  & - & \textbf{0.4606} \\
$R^2_{\mathrm{V,self}}(\textbf{shared})$       & - & 0.2705 \\
$R^2_{\mathrm{V,self}}(\textbf{private})$         & - & 0.2445 \\
\addlinespace[2pt]
$R^2_{\mathrm{T,self}}(\textbf{full global})$ & 0.5737 & - \\
$R^2_{\mathrm{T,self}}(\textbf{full joint})$  & - & \textbf{0.6684} \\
$R^2_{\mathrm{T,self}}(\textbf{shared})$       & - & 0.5088 \\
$R^2_{\mathrm{T,self}}(\textbf{private})$         & - & 0.2550 \\
\midrule
$R^2_{\mathrm{V,cross}}(\textbf{shared})$       & - & \textbf{0.1171} \\
$R^2_{\mathrm{V,cross}}(\textbf{full global})$ & 0.0552 & - \\
$R^2_{\mathrm{V,cross}}(\textbf{full joint})$  & - & 0.1131 \\
\addlinespace[2pt]
$R^2_{\mathrm{T,cross}}(\textbf{shared})$       & - & \textbf{0.4055} \\
$R^2_{\mathrm{T,cross}}(\textbf{full global})$ & 0.3220 & - \\
$R^2_{\mathrm{T,cross}}(\textbf{full joint})$  & - & 0.3971 \\
\bottomrule
\end{tabular}
\caption{\textbf{R$^2$ decomposition under shared only vs shared+private.}
Self R$^2$ is reported for vision (V) and text (T) under four code paths: shared, private, full joint (shared Top-$k$ + private Top-$k$), and full global (global Top-$k$ across all dims; single-dict baseline).
Cross R$^2$ uses the same transport plan $P_{\mathrm{sh}}$ derived from shared-only codes; $\Delta_{\mathrm{leak}}$ measures the change when decoding with full codes vs shared-only under fixed $P_{\mathrm{sh}}$.}
\label{tab:r2_appendix_extensive}
\end{table*}

\begin{figure}[h]
\centering
\caption{Examples of diverse concepts discovered by LUCID on COCO images (Extended version).}
\label{fig:all22}

\includegraphics[width=0.49\linewidth]{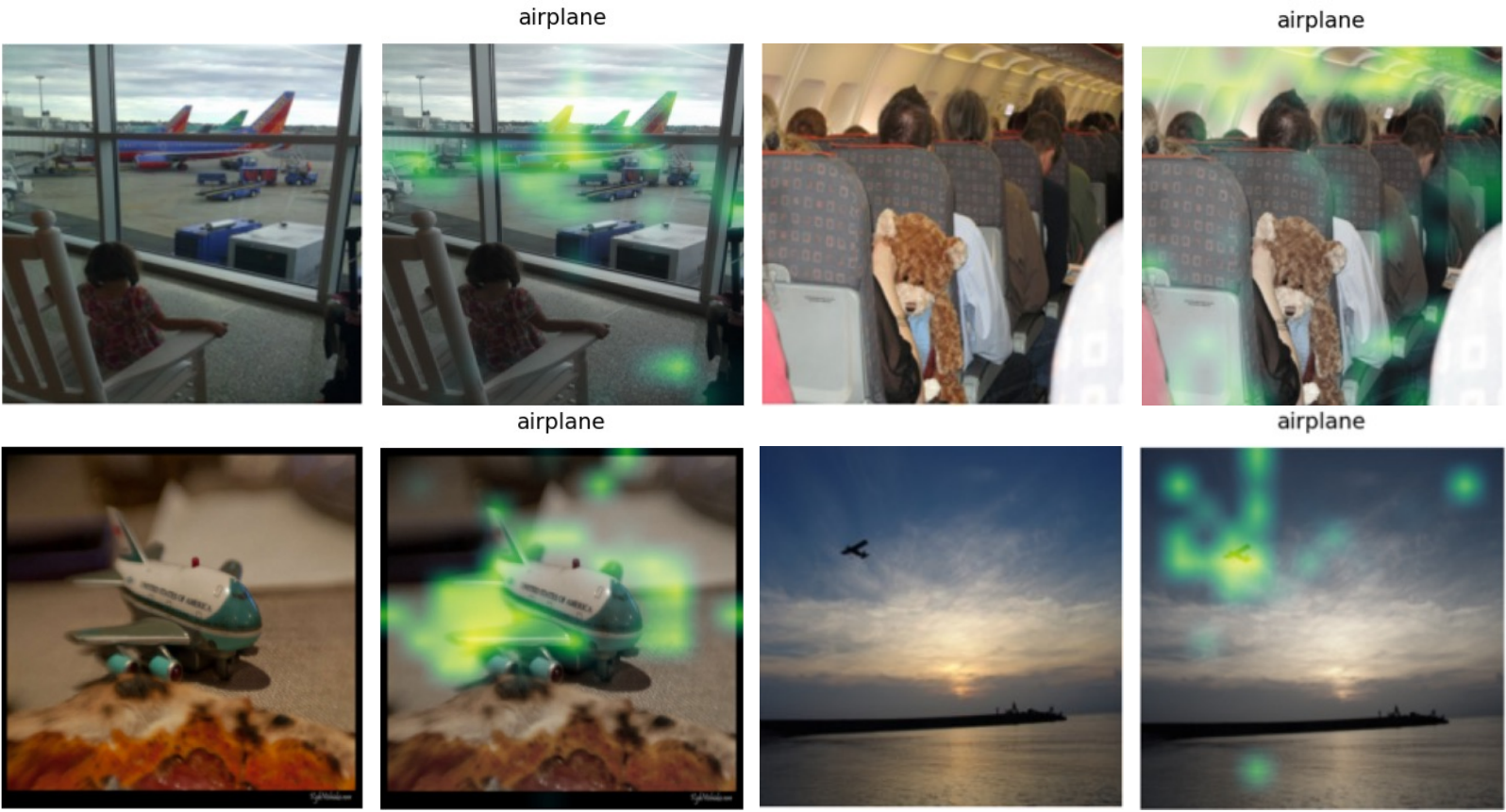}\hfill
\includegraphics[width=0.49\linewidth]{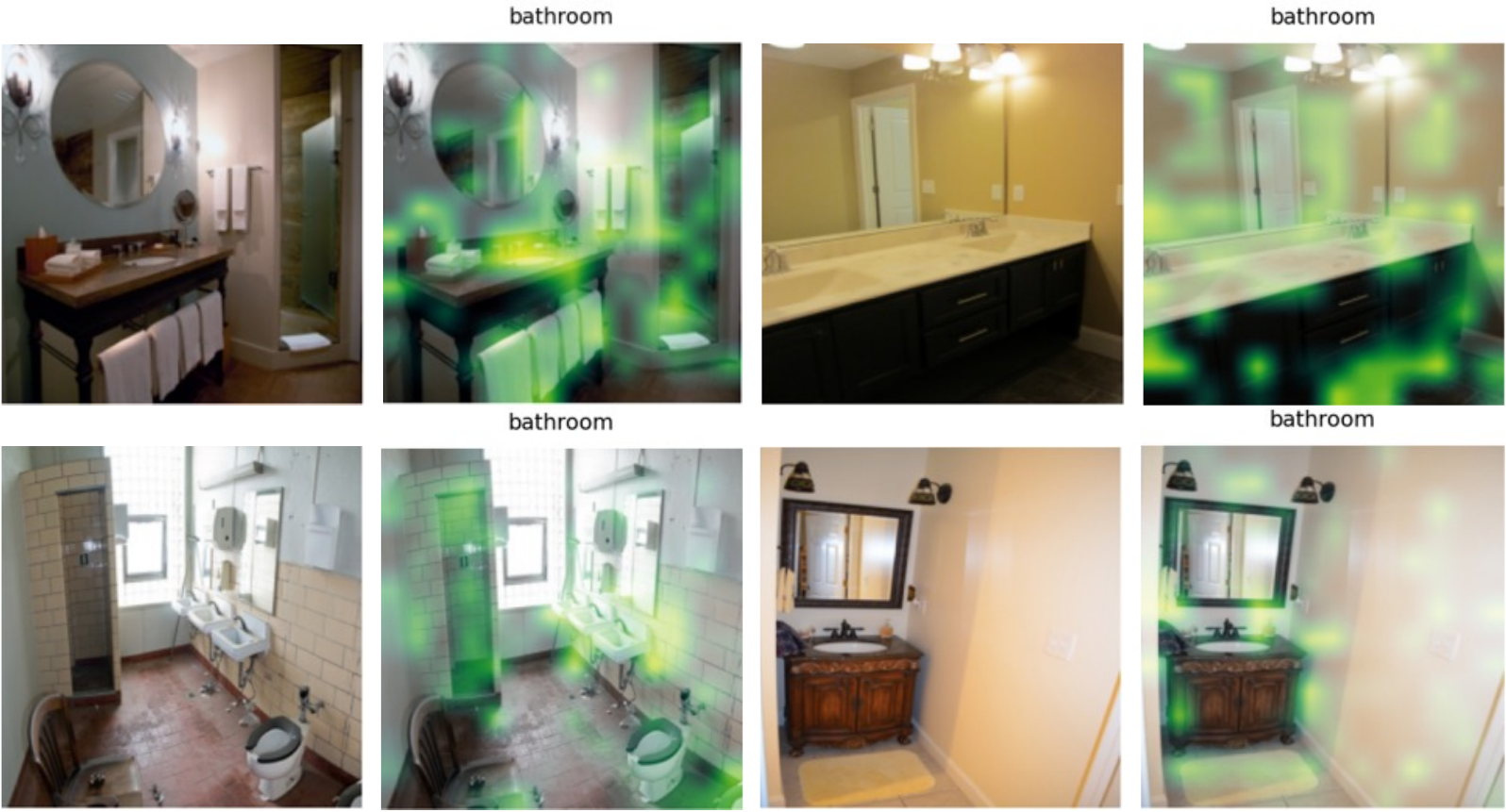}\\[0.6em]
\includegraphics[width=0.49\linewidth]{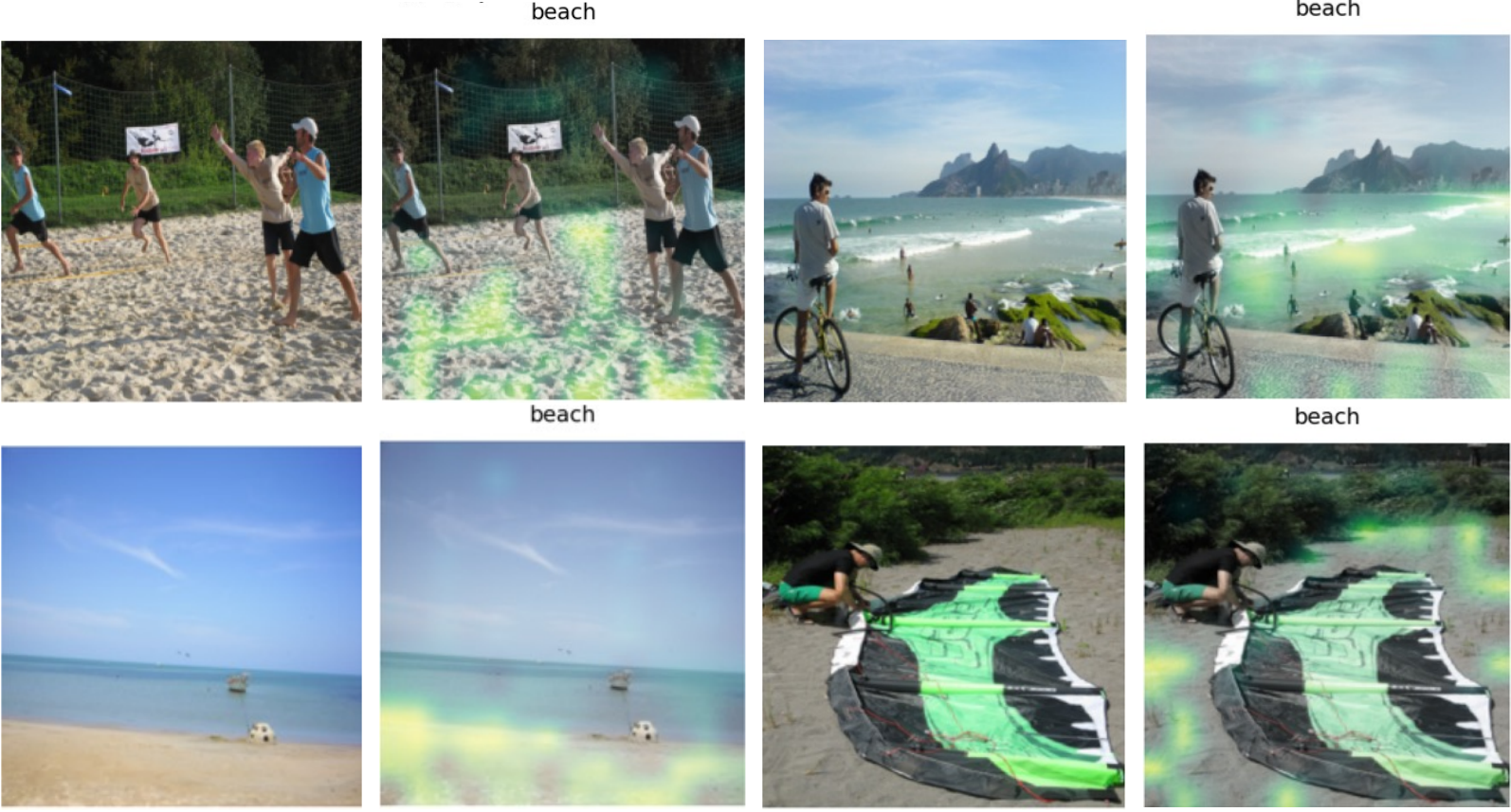}\hfill
\includegraphics[width=0.49\linewidth]{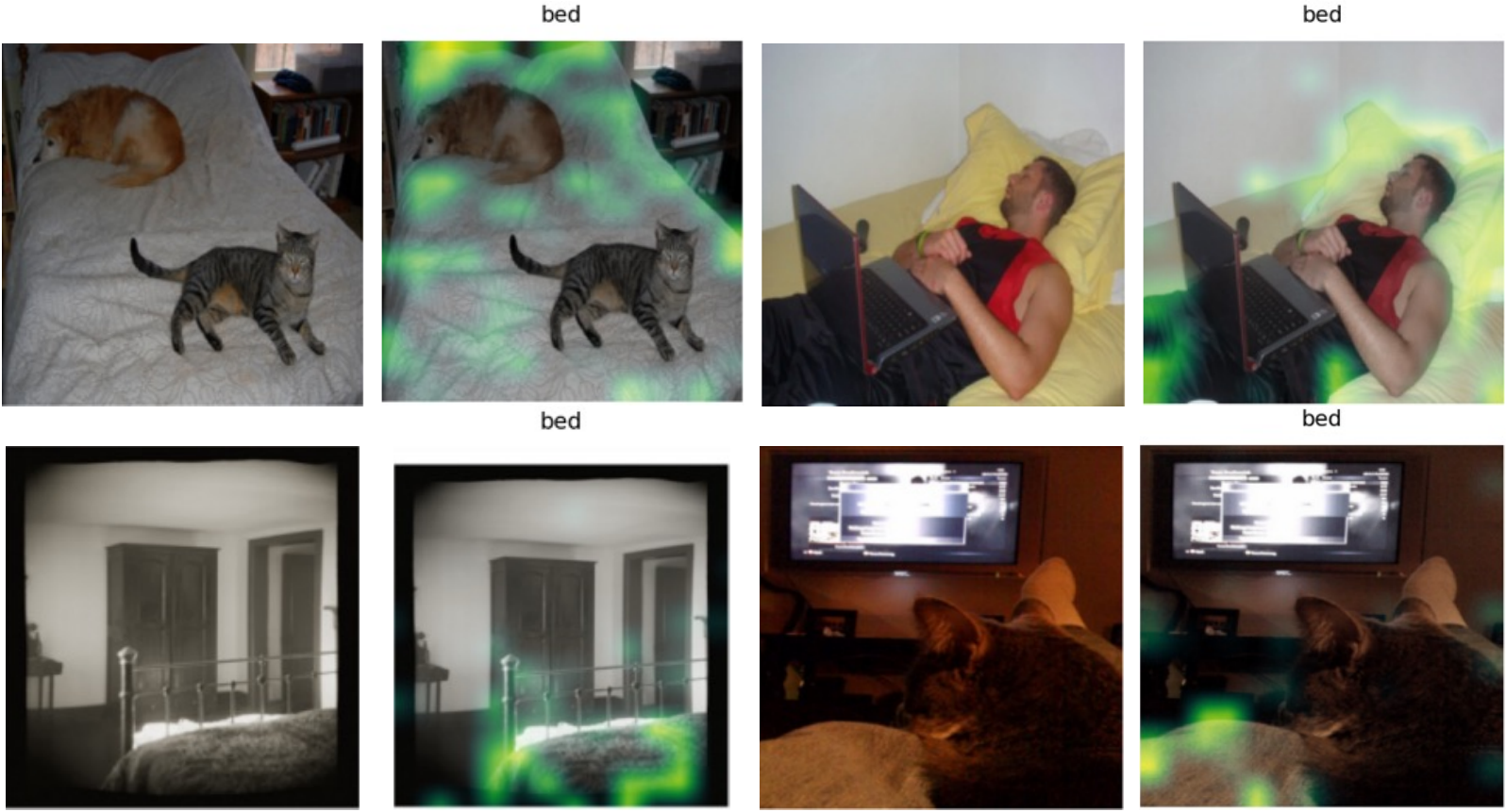}\\[0.6em]
\includegraphics[width=0.49\linewidth]{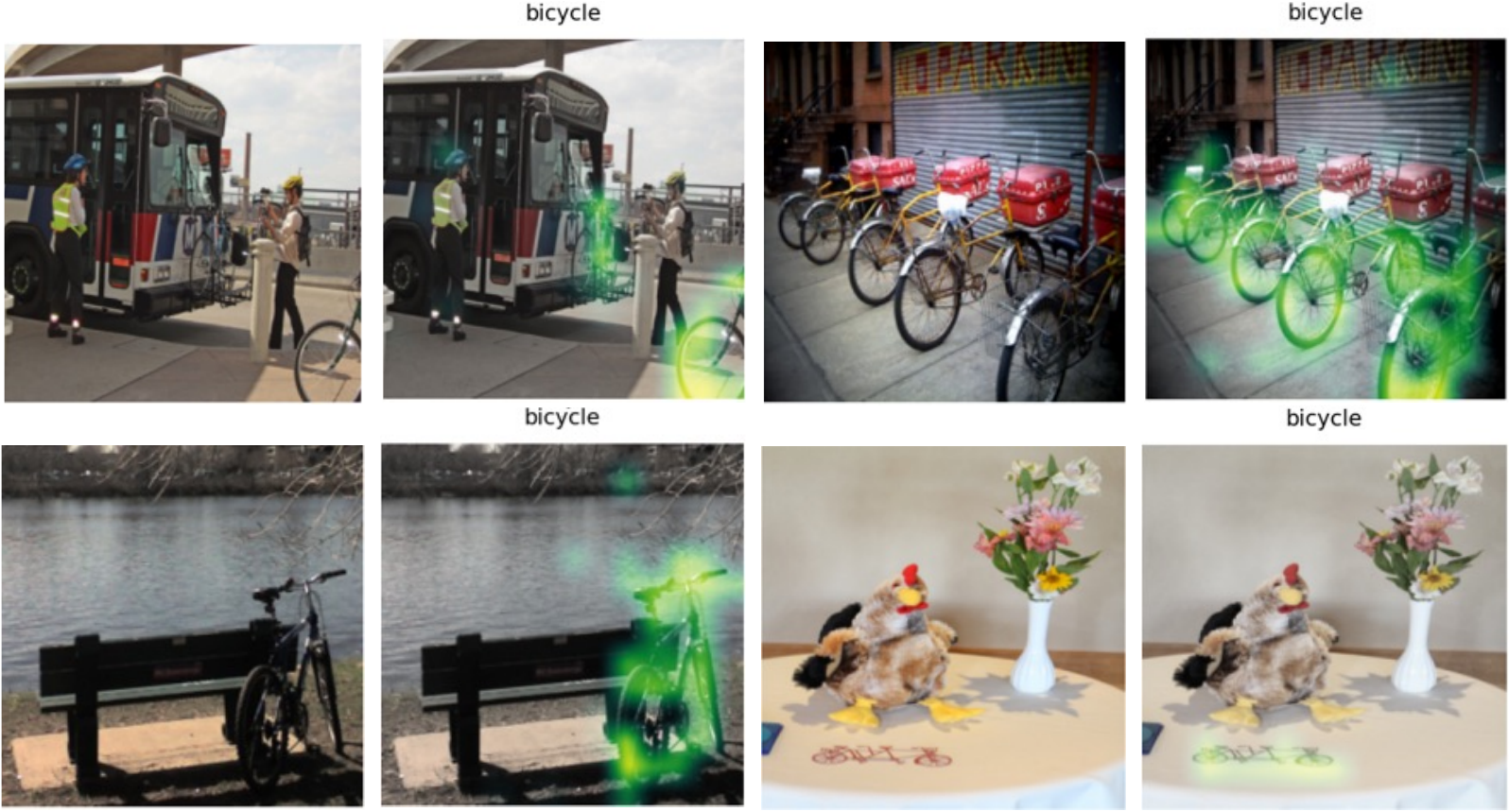}\hfill
\includegraphics[width=0.49\linewidth]{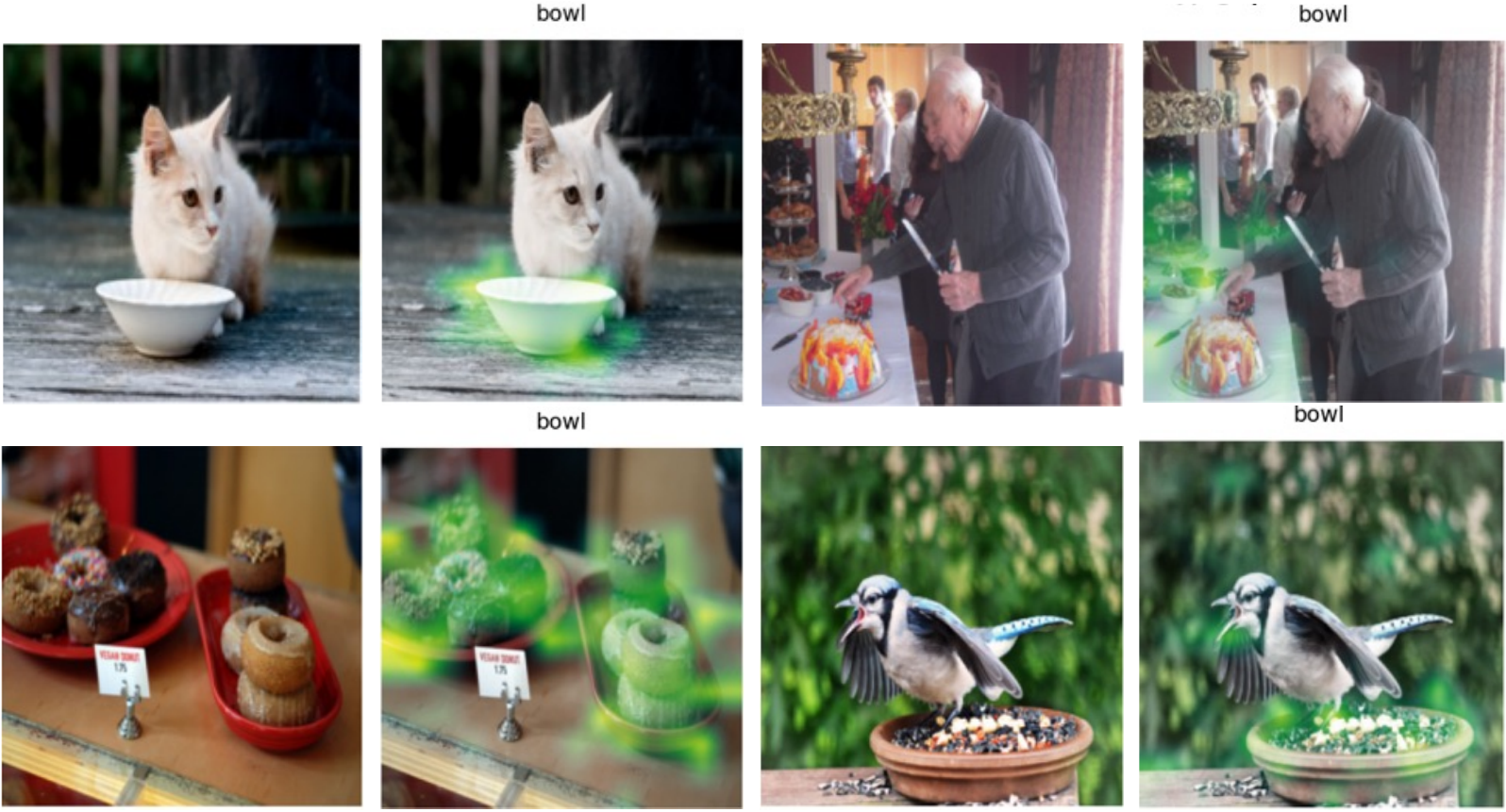}\\[0.6em]
\includegraphics[width=0.49\linewidth]{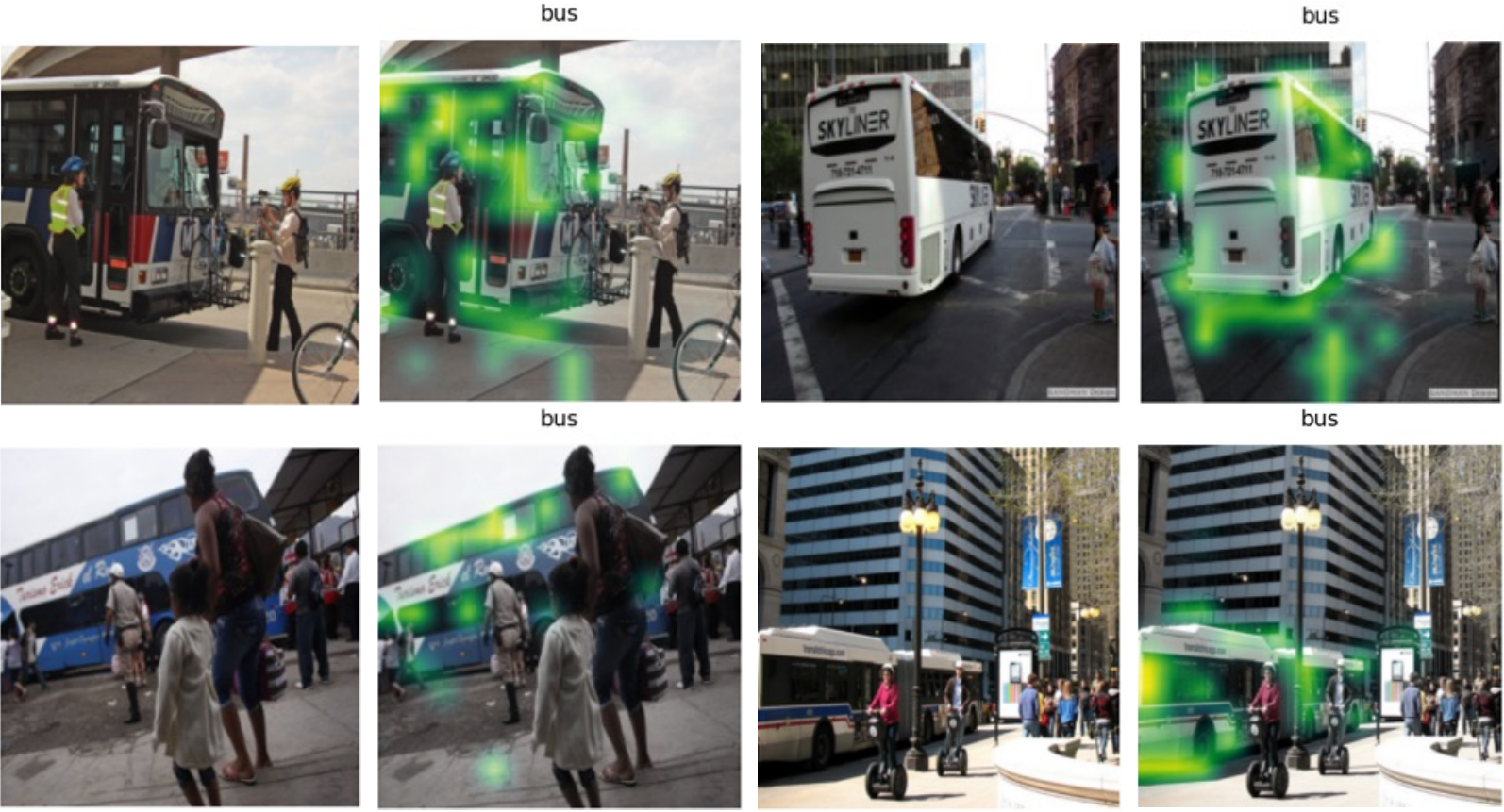}\hfill
\includegraphics[width=0.49\linewidth]{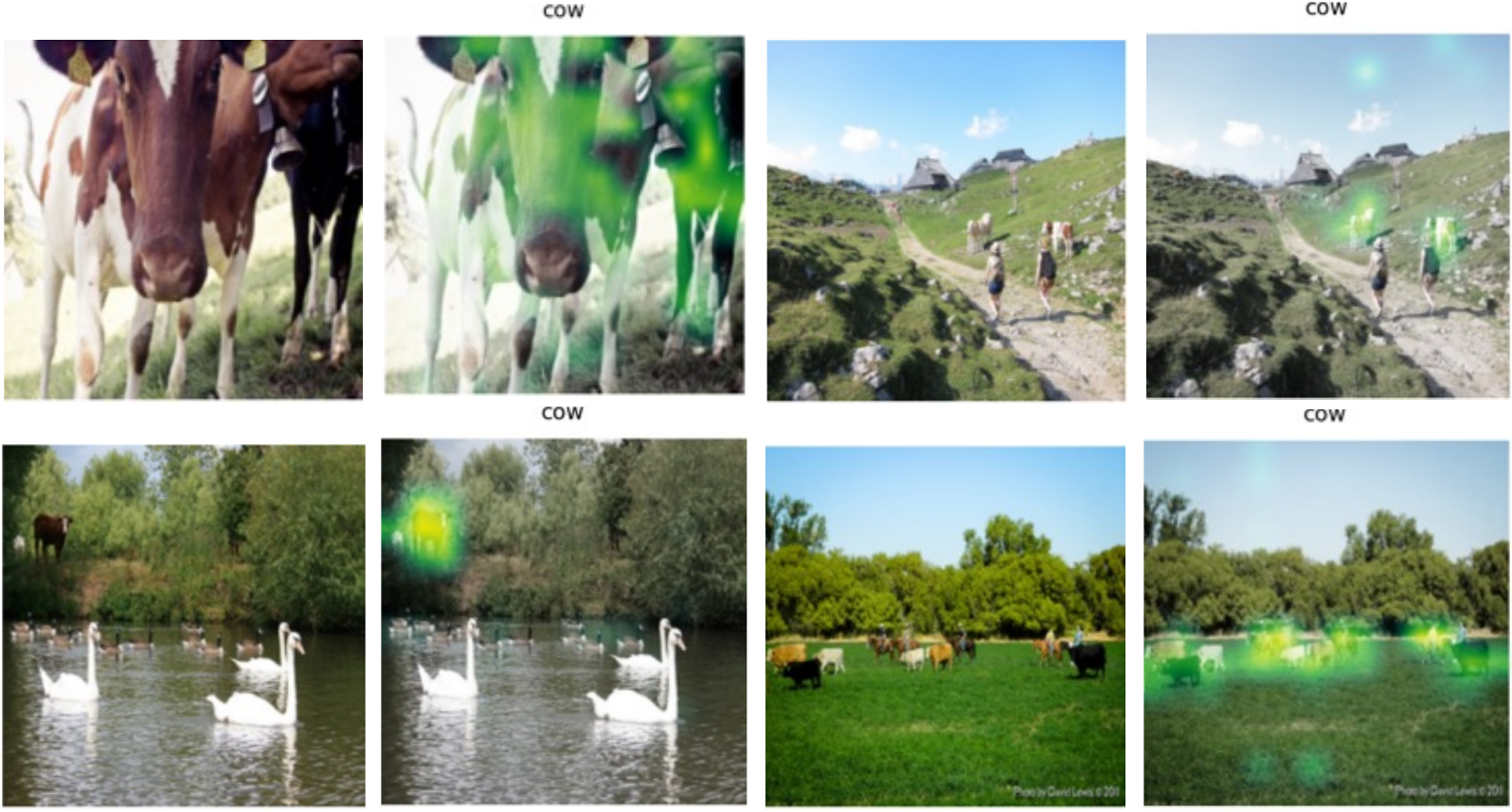}\\[0.6em]
\end{figure}

\begin{figure}[p]\ContinuedFloat
\centering
\captionsetup{list=off} 


\includegraphics[width=0.49\linewidth]{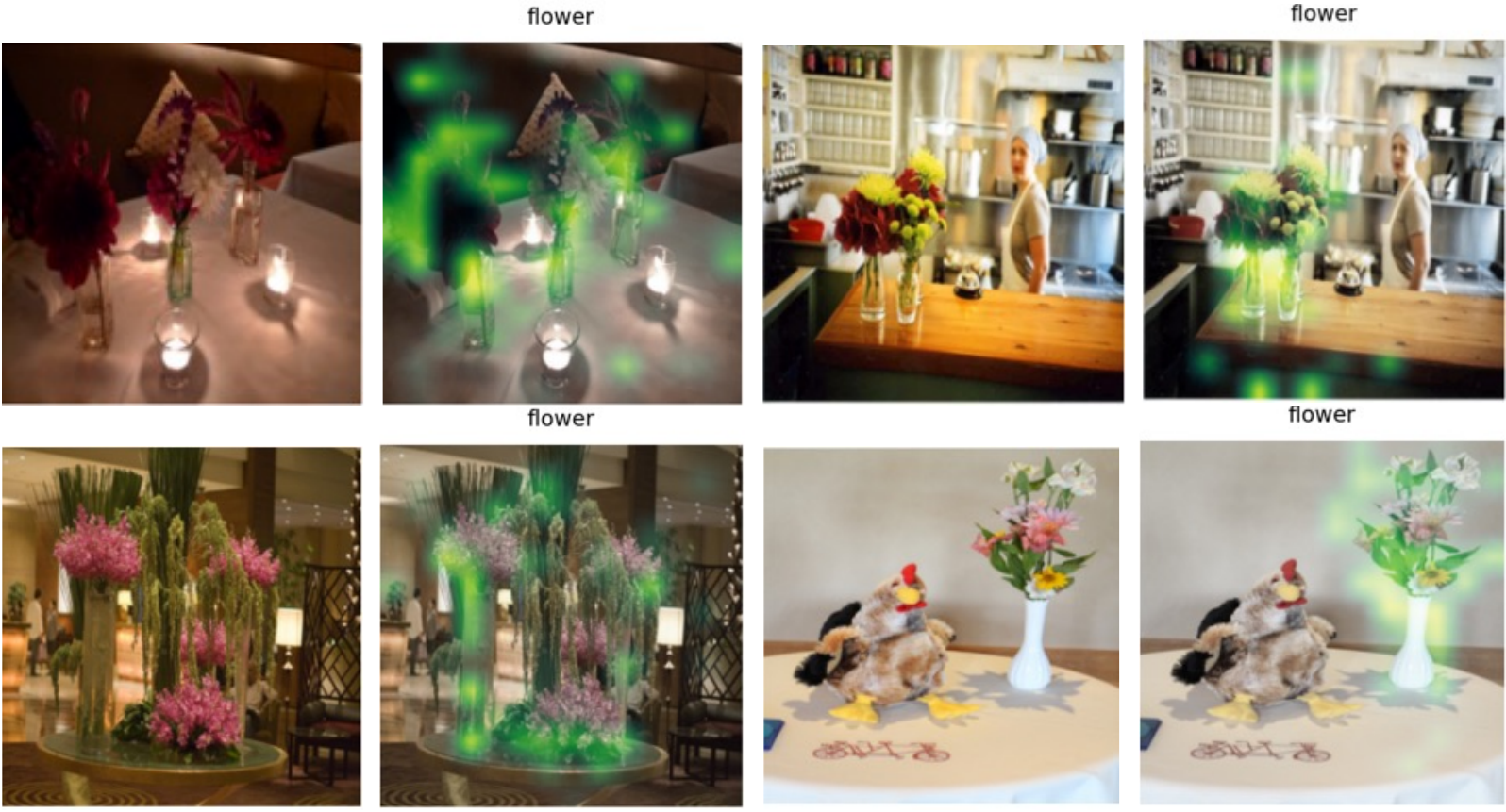}\hfill
\includegraphics[width=0.49\linewidth]{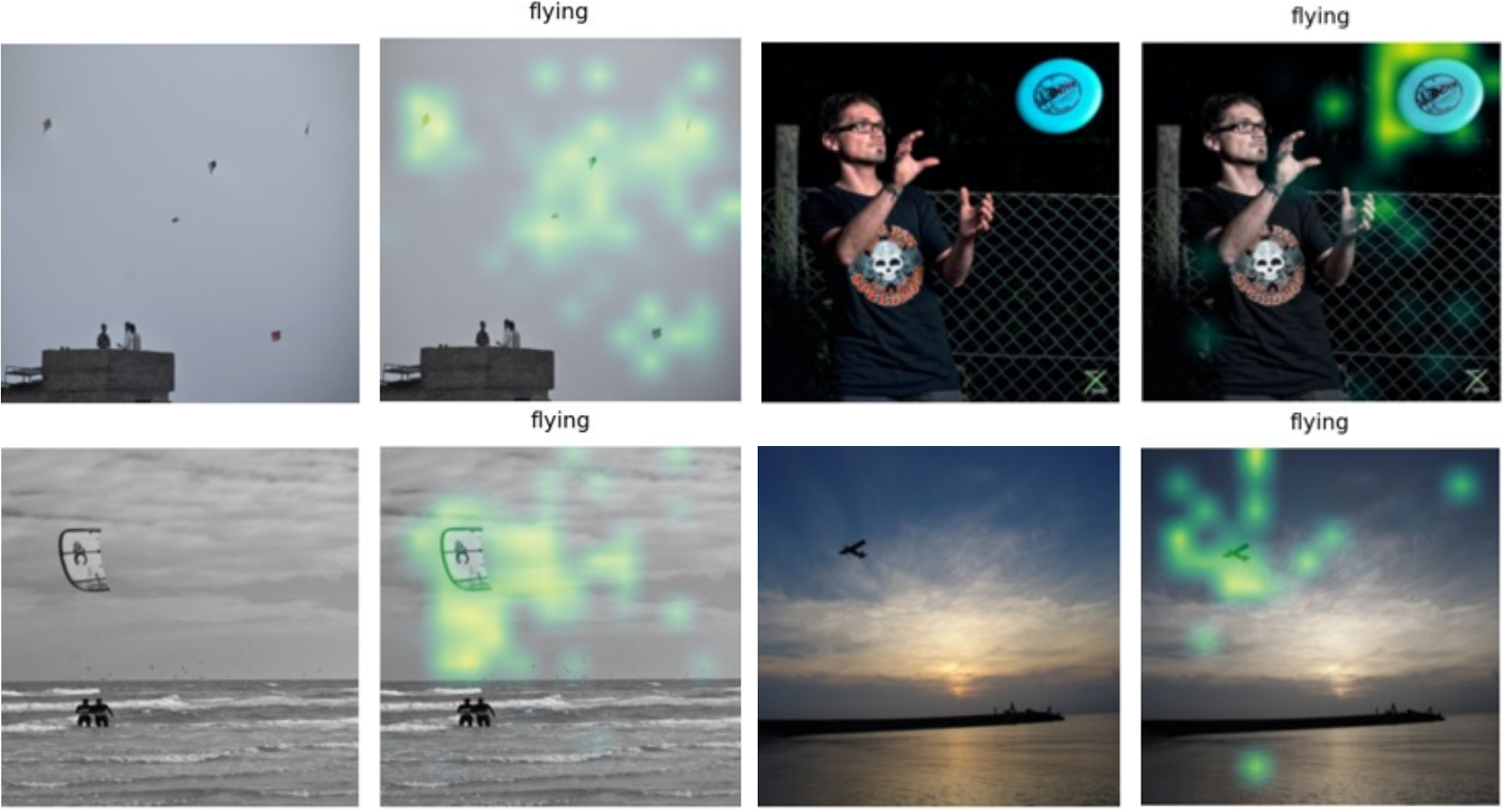}\\[0.6em]
\includegraphics[width=0.49\linewidth]{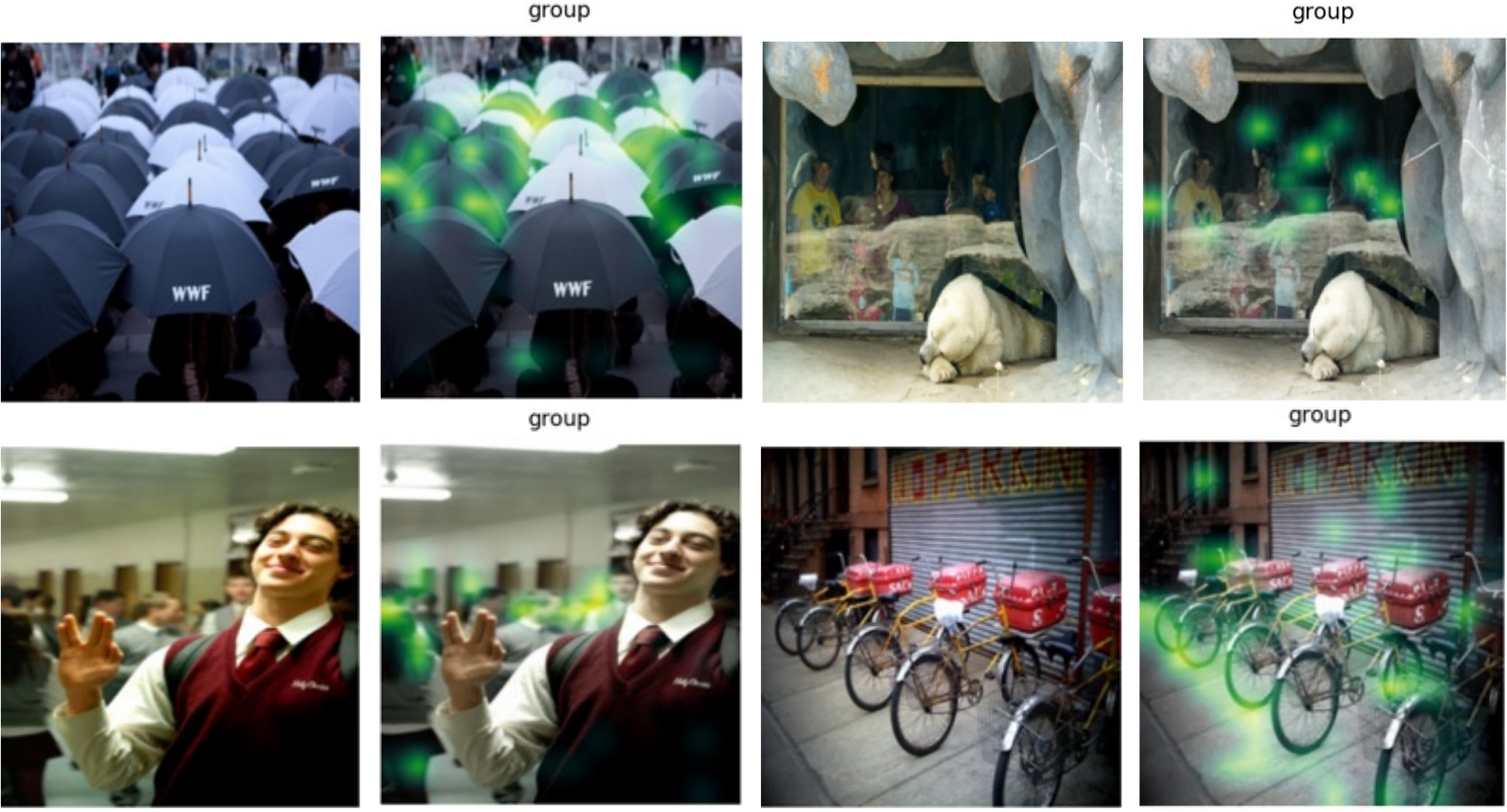}\hfill
\includegraphics[width=0.49\linewidth]{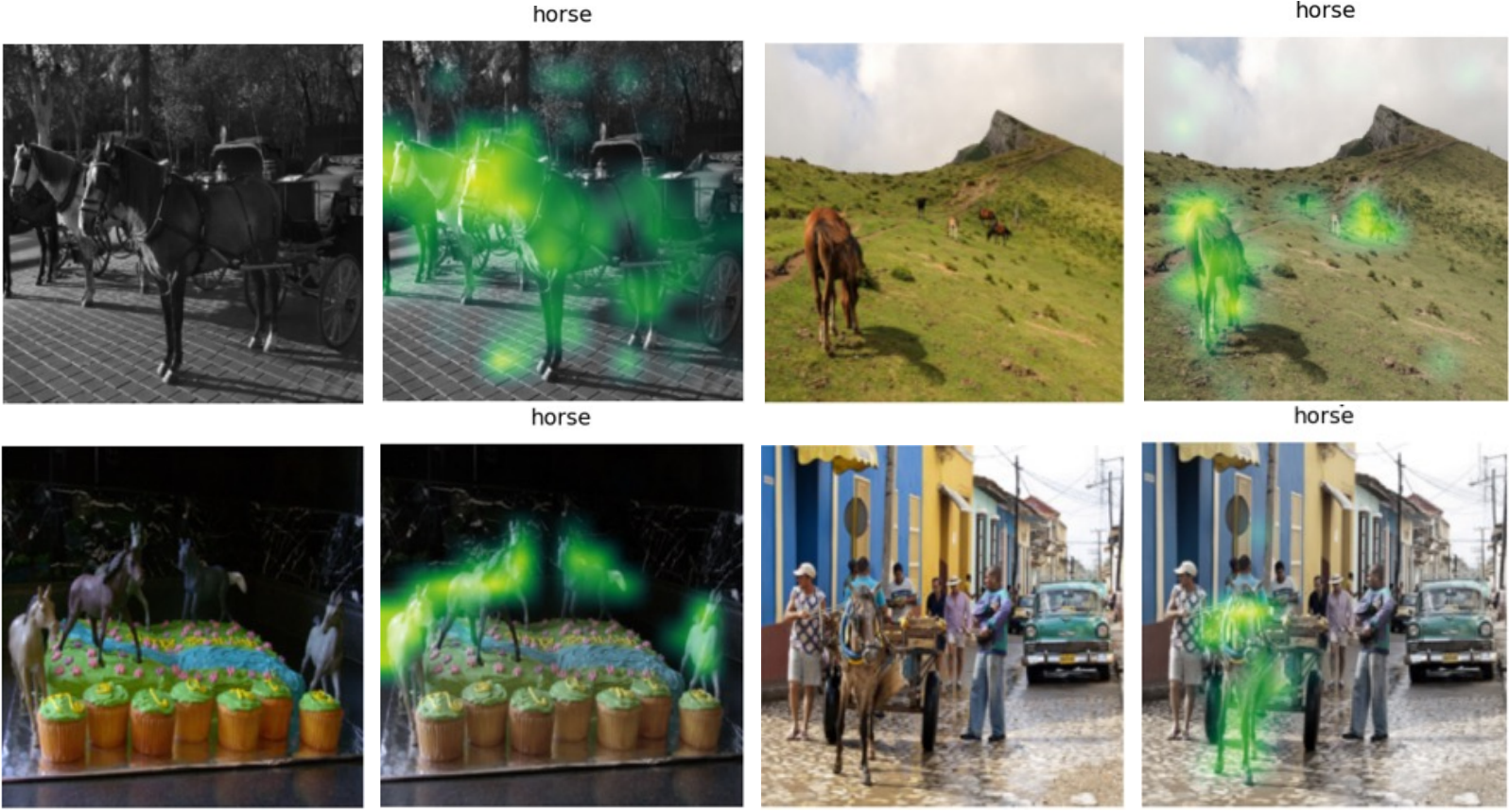}\\[0.6em]
\includegraphics[width=0.49\linewidth]{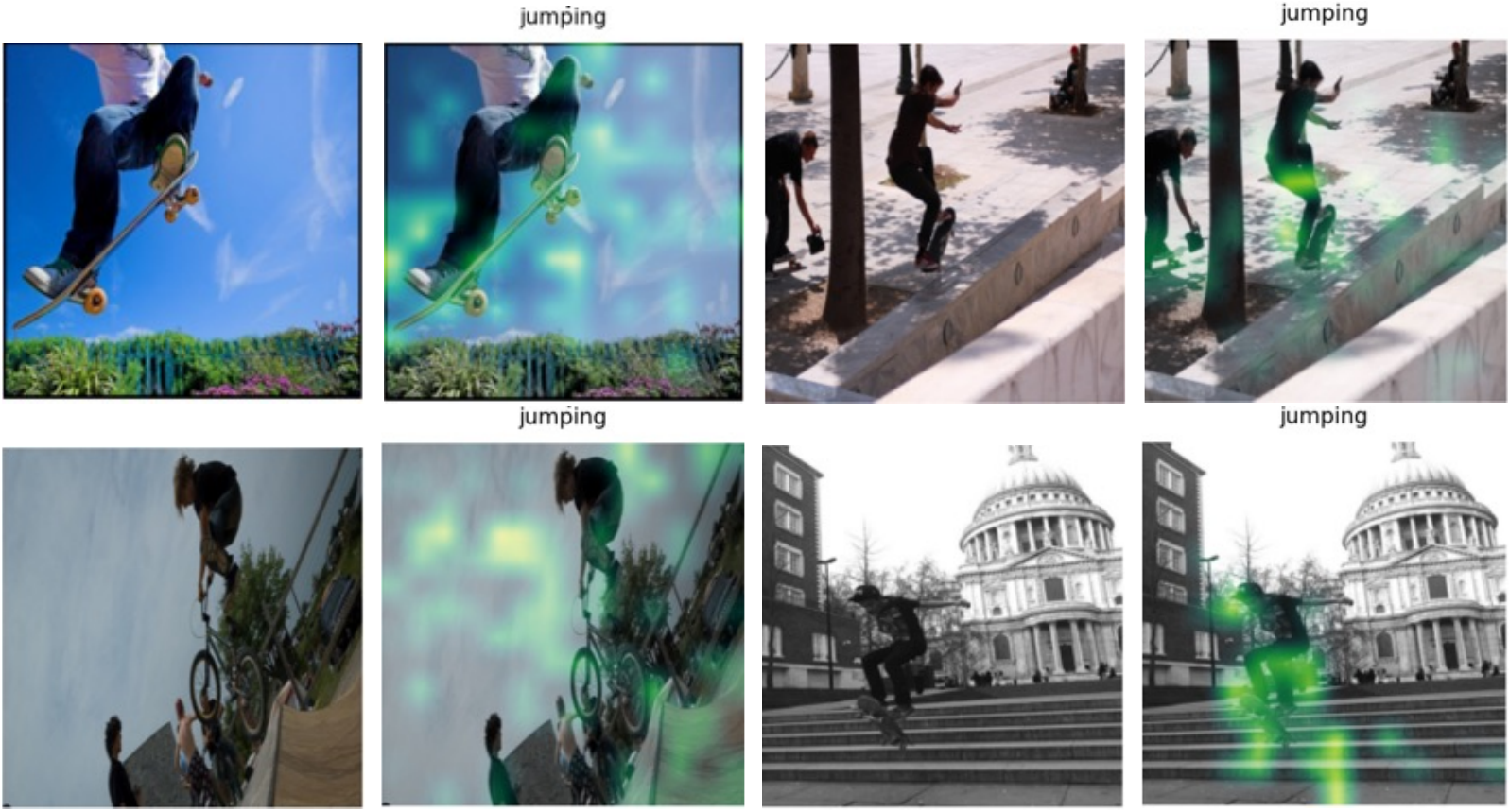}\hfill
\includegraphics[width=0.49\linewidth]{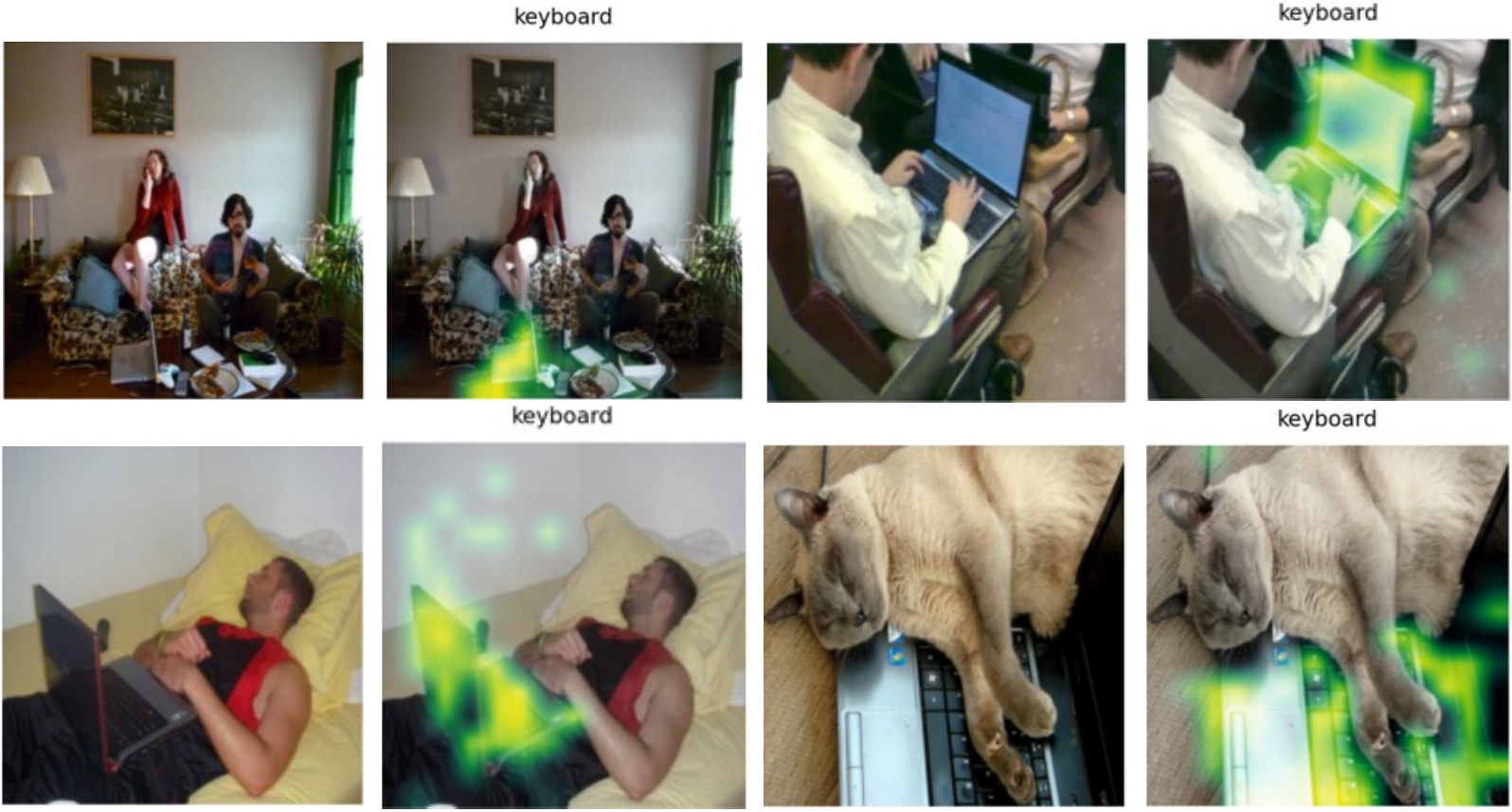}\\[0.6em]
\includegraphics[width=0.49\linewidth]{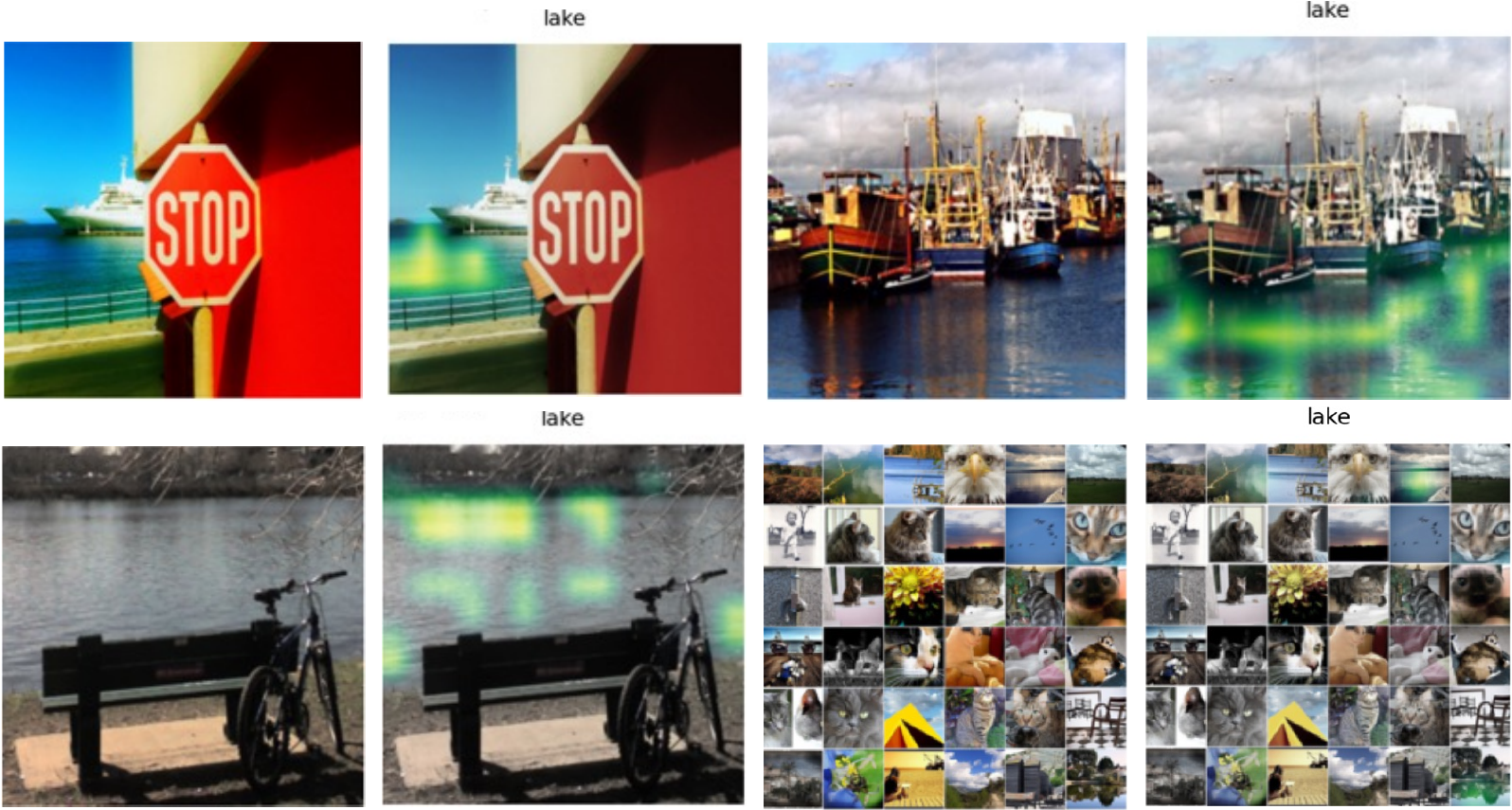}\hfill
\includegraphics[width=0.49\linewidth]{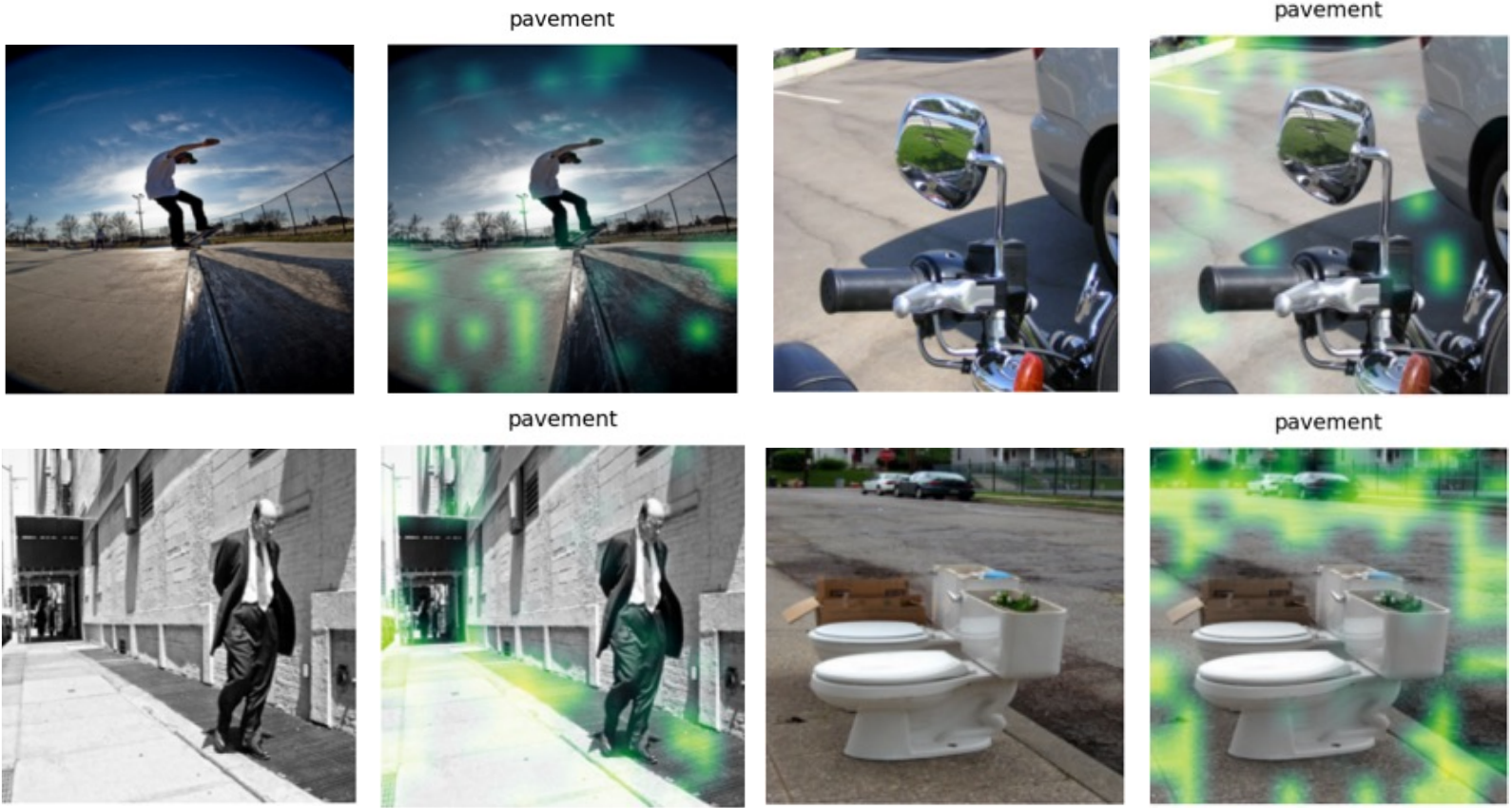}\\[0.6em]
\end{figure}

\begin{figure}[p]\ContinuedFloat
\centering
\captionsetup{list=off}

\includegraphics[width=0.49\linewidth]{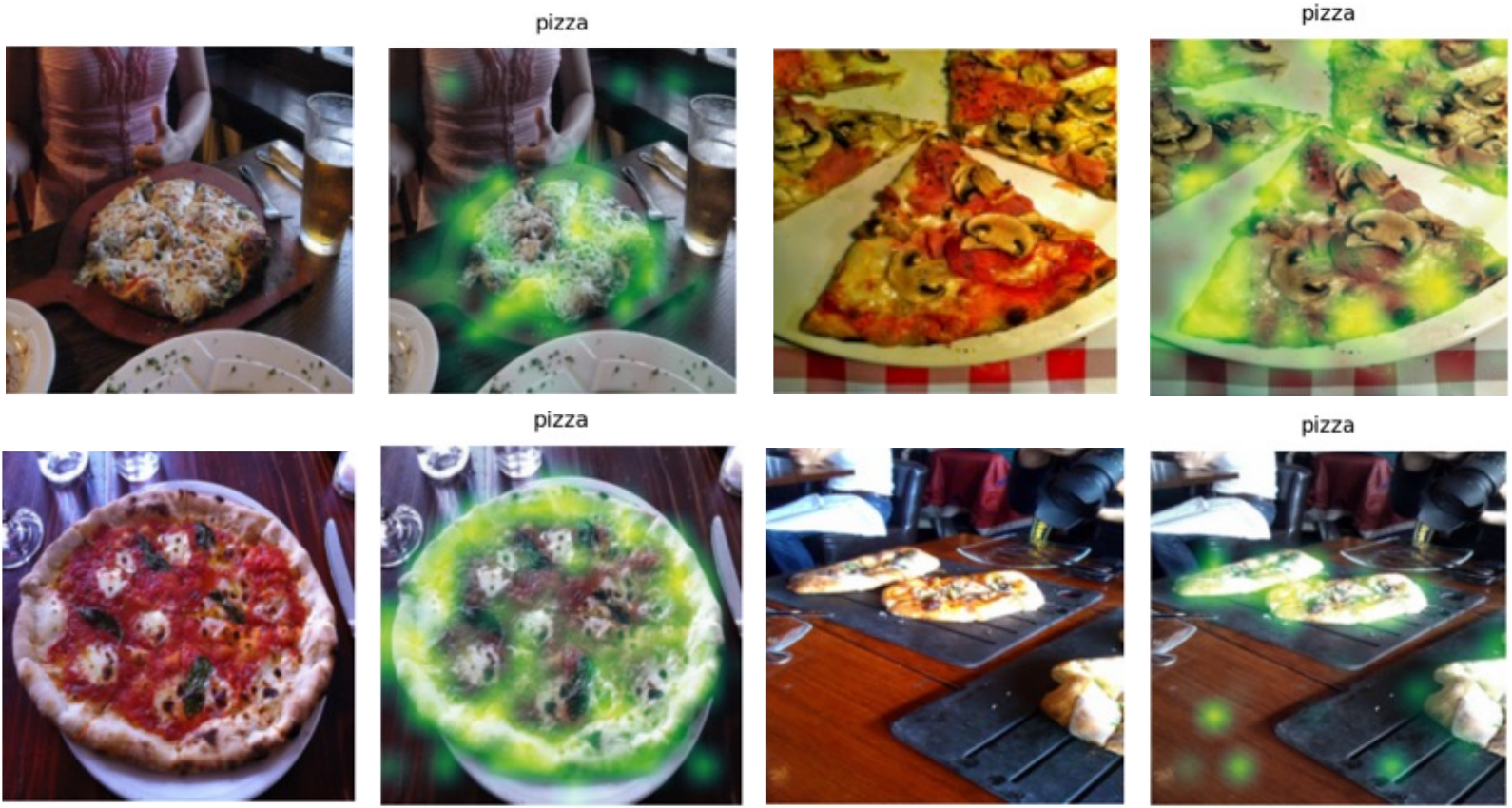}\hfill
\includegraphics[width=0.49\linewidth]{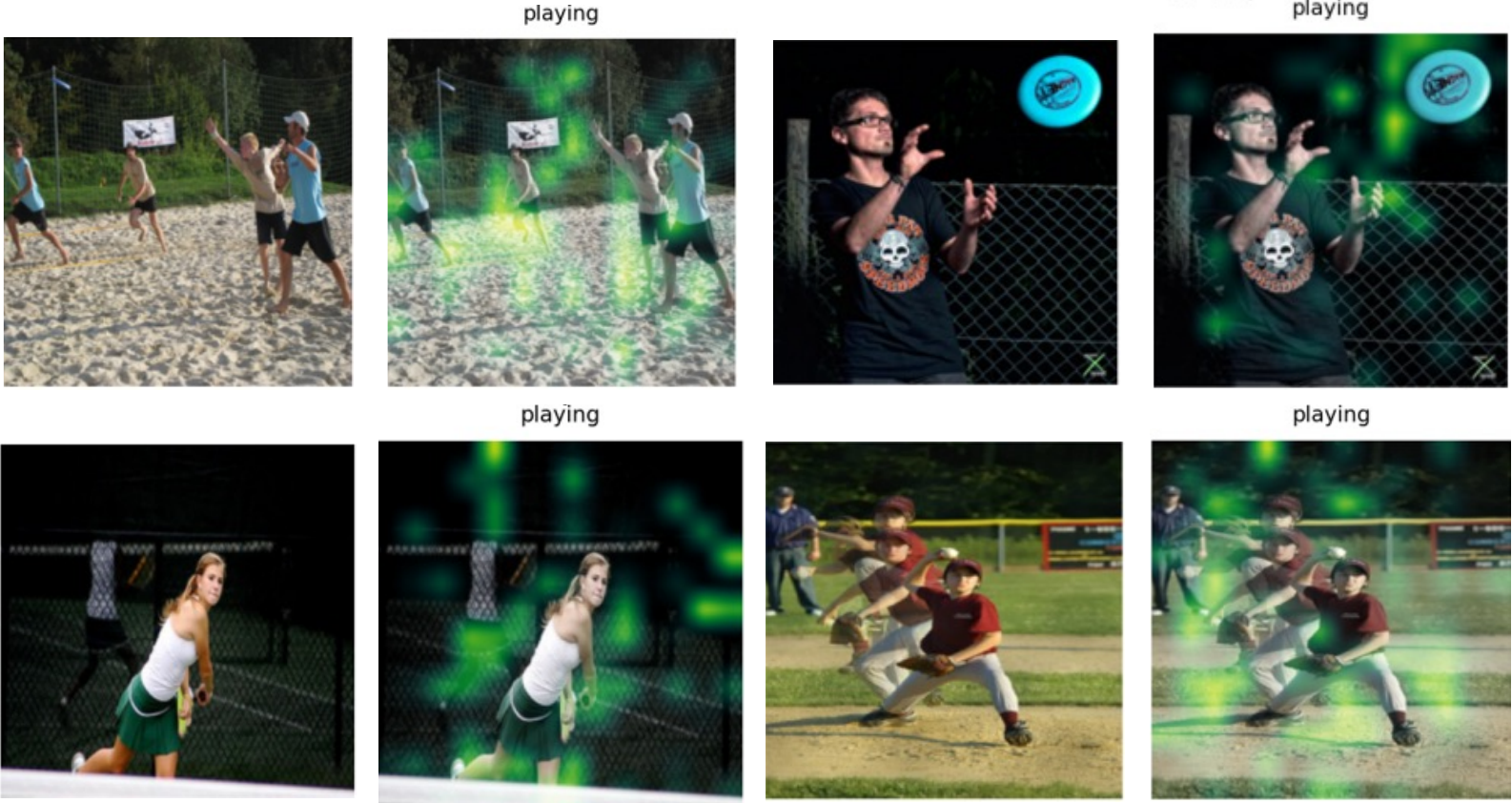}\\[0.6em]
\includegraphics[width=0.49\linewidth]{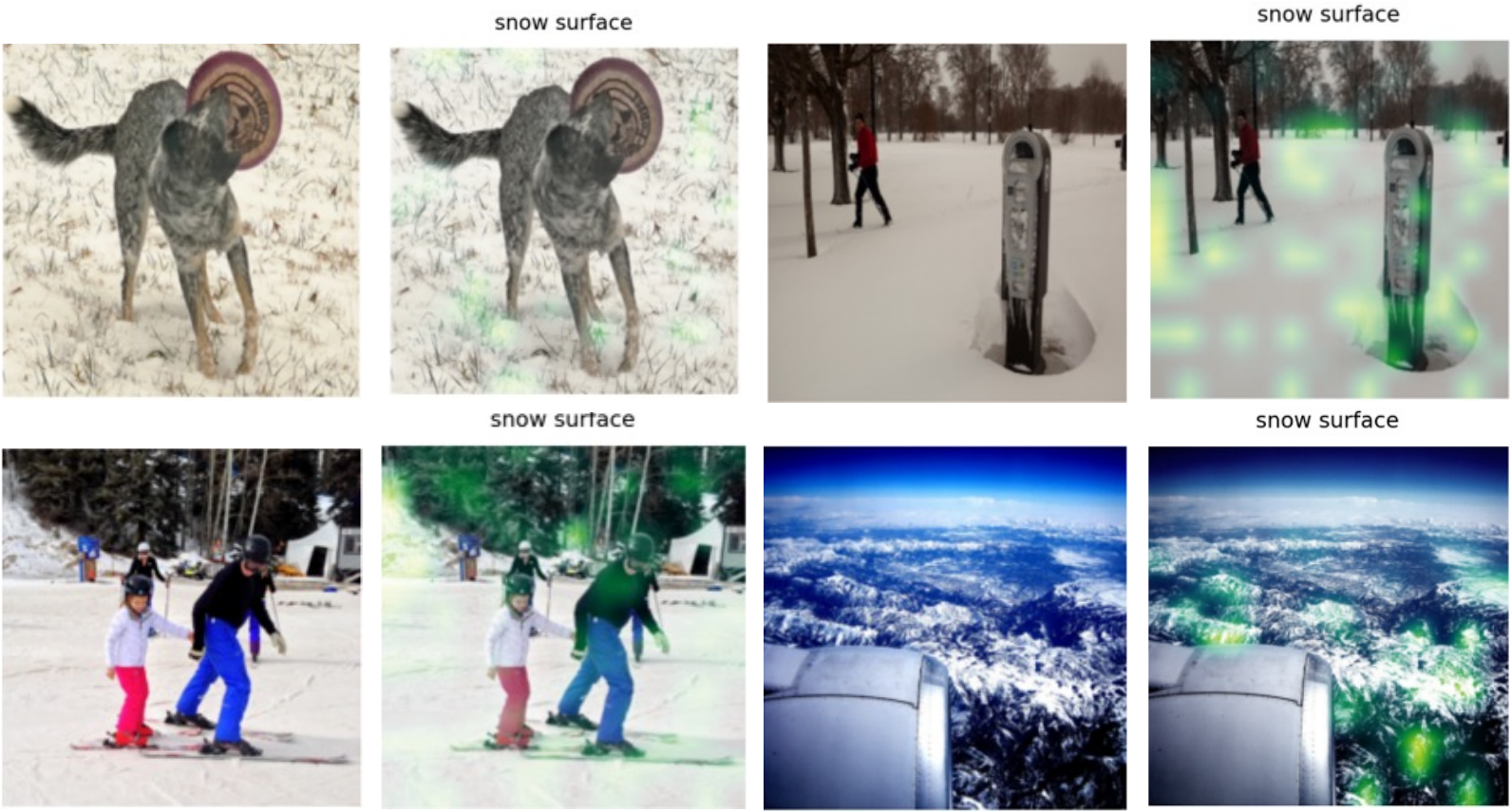}\hfill
\includegraphics[width=0.49\linewidth]{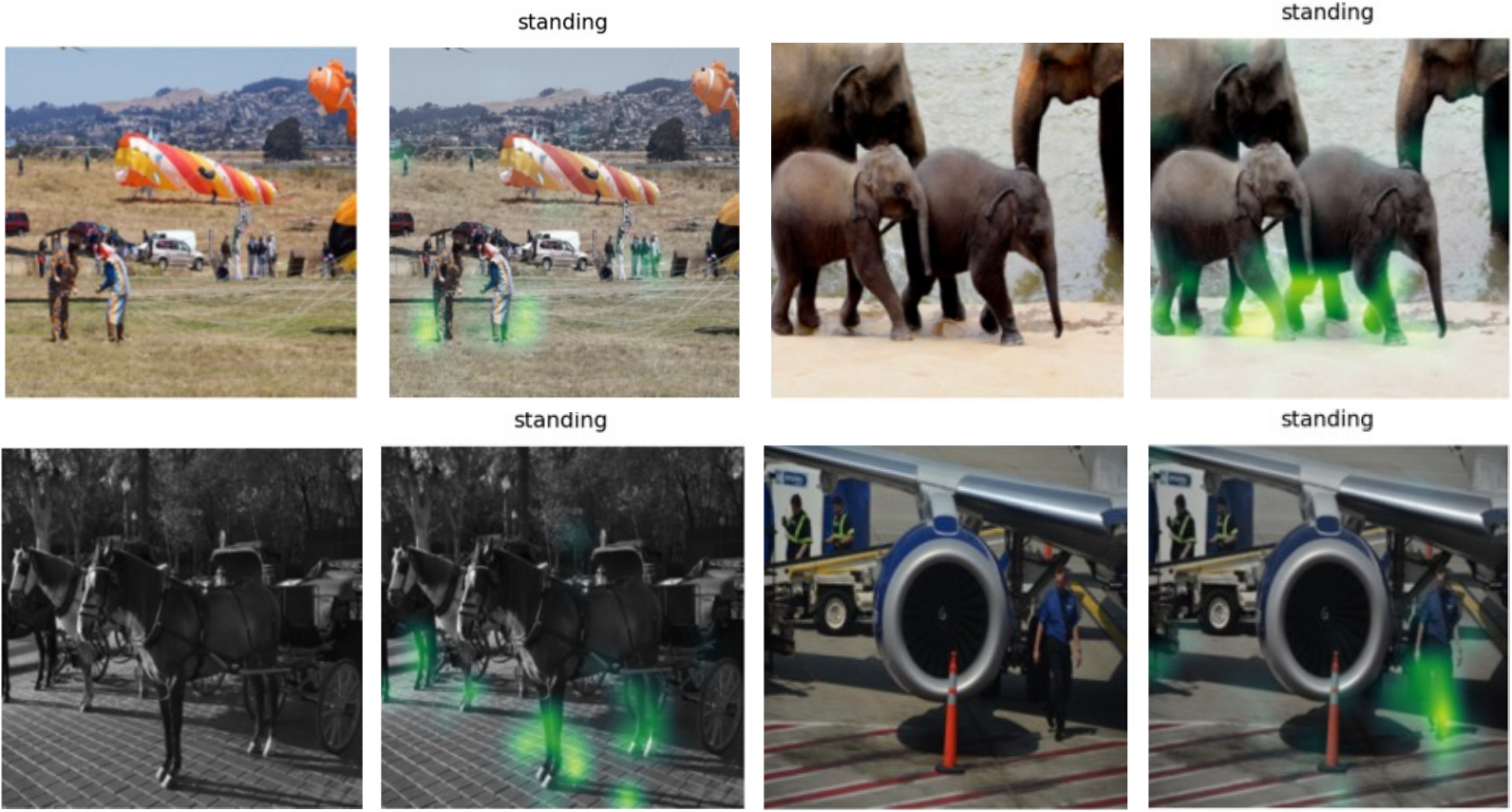}\\[0.6em]
\includegraphics[width=0.49\linewidth]{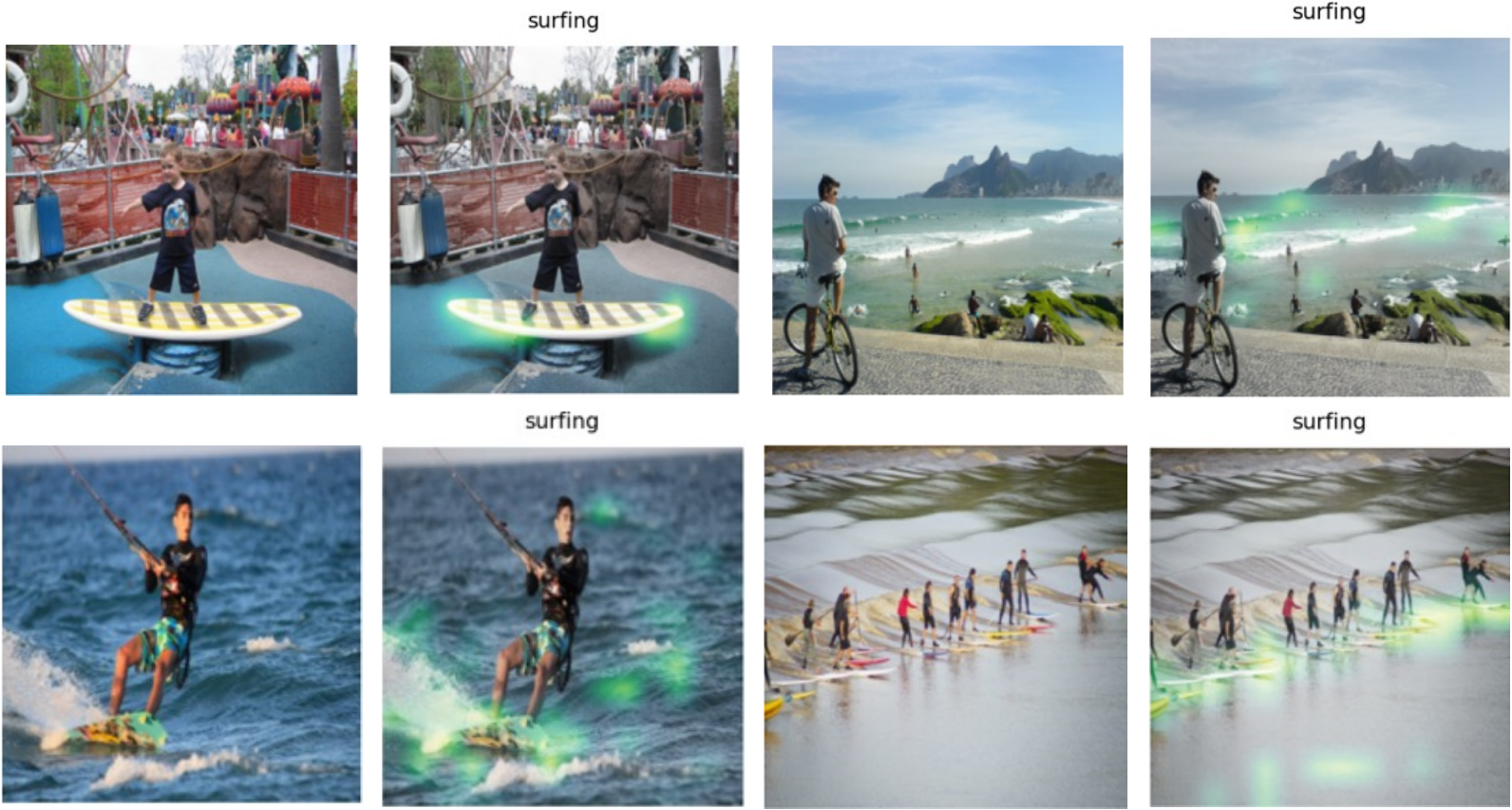}\hfill
\includegraphics[width=0.49\linewidth]{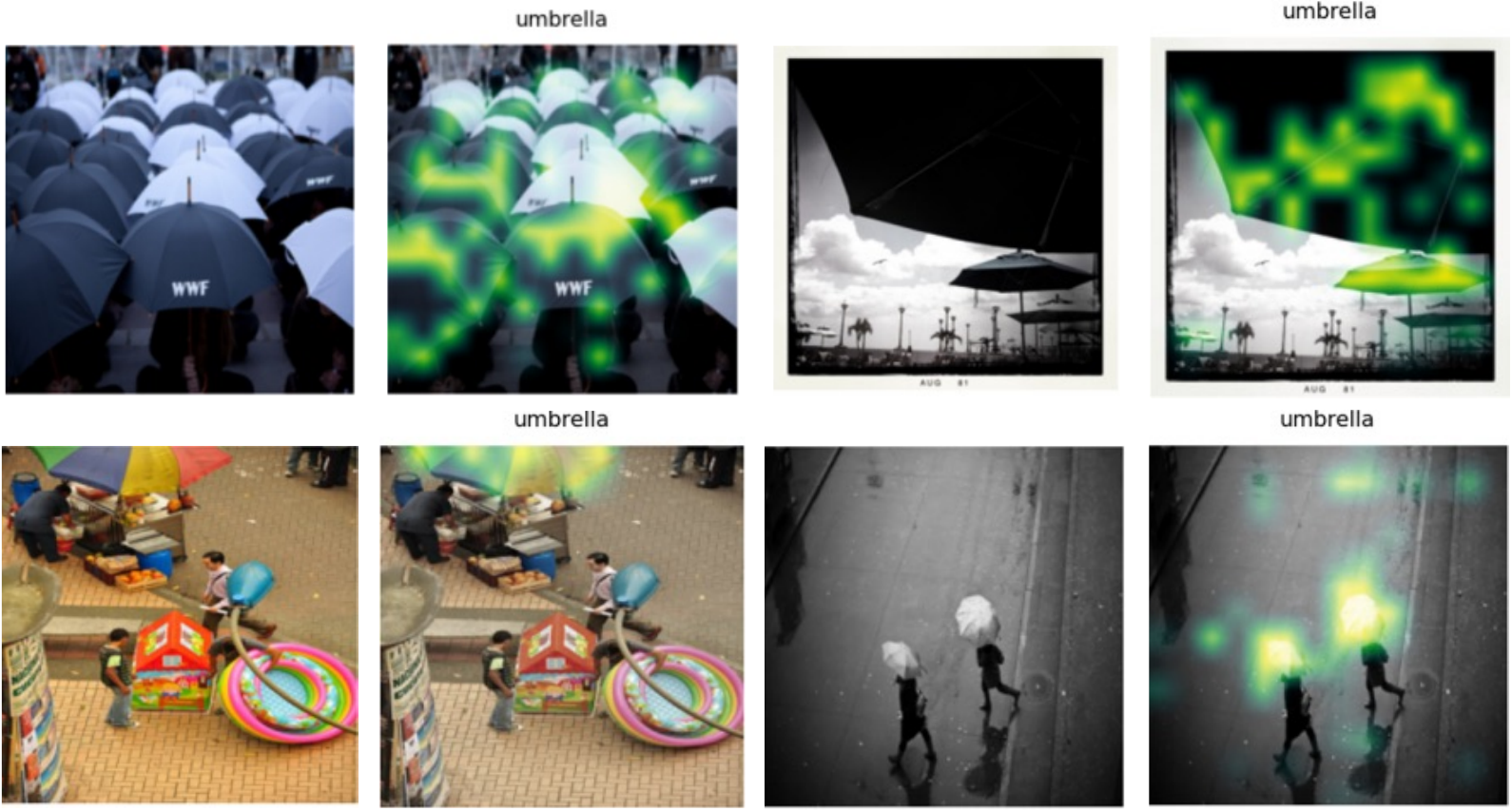}\\[0.6em]
\end{figure}

\subsection{Omitted Proofs}
\subsubsection{Proof of Theorem~\ref{thm:sinkhorn}}
\begin{proof}
We proceed via the method of Lagrange multipliers.

\textbf{Step 1: Lagrangian formulation.}
The problem involves equality constraints for marginals and non-negativity constraints $\Pi_{ij} \ge 0$. Introducing dual variables $\lambda \in \mathbb{R}^n$ and $\mu \in \mathbb{R}^m$ for the marginal constraints, the Lagrangian is:
\begin{equation*}
\mathcal{L}(\Pi, \lambda, \mu) = \sum_{i,j} C_{ij} \Pi_{ij} + \varepsilon \sum_{i,j} \Pi_{ij} \log \Pi_{ij} + \sum_i \lambda_i \left( r_i - \sum_j \Pi_{ij} \right) + \sum_j \mu_j \left( c_j - \sum_i \Pi_{ij} \right).
\end{equation*}

\textbf{Step 2: First-order optimality.}
Due to the singularity of the gradient of the entropy term at zero ($\lim_{x \to 0^+} \log x = -\infty$), the optimal solution is strictly positive ($\Pi_{ij}^\star > 0$). Consequently, the non-negativity constraints are inactive. Taking the partial derivative with respect to $\Pi_{ij}$ and setting it to zero yields:
\begin{equation*}
\frac{\partial \mathcal{L}}{\partial \Pi_{ij}} = C_{ij} + \varepsilon (\log \Pi_{ij} + 1) - \lambda_i - \mu_j = 0.
\end{equation*}
Rearranging terms:
\begin{equation*}
\log \Pi_{ij} = \frac{\lambda_i - \varepsilon}{\varepsilon} + \frac{\mu_j}{\varepsilon} - \frac{C_{ij}}{\varepsilon}.
\end{equation*}

\textbf{Step 3: Exponentiating.}
Exponentiating both sides results in:
\begin{equation*}
\Pi_{ij}^\star = \exp\left( \frac{\lambda_i - \varepsilon}{\varepsilon} \right) \cdot \exp\left( \frac{\mu_j}{\varepsilon} \right) \cdot \exp\left( -\frac{C_{ij}}{\varepsilon} \right).
\end{equation*}
By defining the scaling variables $u_i := \exp\left( \frac{\lambda_i - \varepsilon}{\varepsilon} \right)$, $v_j := \exp\left( \frac{\mu_j}{\varepsilon} \right)$, and the kernel $K_{ij} := \exp\left( -\frac{C_{ij}}{\varepsilon} \right)$, we obtain the factorization:
\begin{equation*}
\Pi_{ij}^\star = u_i \, K_{ij} \, v_j,
\end{equation*}
which corresponds to $\Pi^\star = \mathrm{diag}(u) \, K \, \mathrm{diag}(v)$ in matrix form.

\textbf{Step 4: Determining scaling vectors.}
Substituting $\Pi^\star$ into the marginal constraints:
\begin{align*}
\text{Row constraint:} \quad & \Pi^\star \mathbf{1}_m = r \implies \mathrm{diag}(u) K v = r \implies u = r \oslash (Kv). \\
\text{Column constraint:} \quad & (\Pi^\star)^\top \mathbf{1}_n = c \implies \mathrm{diag}(v) K^\top u = c \implies v = c \oslash (K^\top u).
\end{align*}

\textbf{Step 5: Uniqueness.}
The objective function is strictly convex on the transportation polytope due to the entropy term. Thus, the optimal plan $\Pi^\star$ is unique. Since $K$ is strictly positive, the scaling vectors $(u, v)$ exist and are unique up to a scalar factor (Sinkhorn's Theorem).
\end{proof}

\begin{remark}
The scaling vectors relate to Kantorovich dual potentials via $u_i = \exp(\phi_i / \varepsilon)$, $v_j = \exp(\psi_j / \varepsilon)$. The ambiguity $(u, v) \mapsto (\alpha u, v/\alpha)$ corresponds to the gauge freedom $(\phi, \psi) \mapsto (\phi + t, \psi - t)$.
\end{remark}

\subsubsection{Linear Convergence of Sinkhorn Algorithm}
\begin{corollary}[Sinkhorn Algorithm]
\label{cor:sinkhorn}
The scaling vectors $(u, v)$ in Theorem~\ref{thm:sinkhorn} can be computed via the Sinkhorn-Knopp iteration. Initializing $v^{(0)} = \mathbf{1}$, the updates for $t = 0, 1, 2, \ldots$ are:
\begin{equation}
u^{(t+1)} = r_n \oslash (K v^{(t)}), \quad v^{(t+1)} = c_n \oslash (K^\top u^{(t+1)}),
\end{equation}
where $\oslash$ denotes elementwise division. This iteration converges linearly in the Hilbert projective metric, with contraction rate $\kappa = \tau(K)^2 < 1$.
\end{corollary}
\begin{proof}
We establish convergence via the Birkhoff-Hopf theorem on contractions in Hilbert's projective metric.

\textbf{Step 1: Hilbert projective metric.}
For vectors $x, y \in \mathbb{R}_{++}^n$, the Hilbert projective metric is defined as:
\begin{equation*}
d_H(x, y) = \log \left( \max_i \frac{x_i}{y_i} \right) - \log \left( \min_j \frac{x_j}{y_j} \right) = \log \left( \max_{i,j} \frac{x_i y_j}{x_j y_i} \right).
\end{equation*}
This defines a pseudo-metric on $\mathbb{R}_{++}^n$ satisfying $d_H(x, y) = 0$ iff $x = \alpha y$ for some $\alpha > 0$. Consequently, $d_H$ induces a true metric on the projective cone $\mathbb{R}_{++}^n / \mathbb{R}_{++}$.

\textbf{Step 2: Birkhoff's contraction theorem.}
For a matrix $A \in \mathbb{R}_{++}^{n \times m}$, define its projective diameter:
\begin{equation*}
\Delta(A) = \max_{i,j,k,l} \log \frac{A_{ik} A_{jl}}{A_{il} A_{jk}},
\end{equation*}
and the Birkhoff contraction coefficient $\tau(A) = \tanh(\Delta(A)/4) \in [0, 1)$. Birkhoff's theorem states that for any $x, y \in \mathbb{R}_{++}^m$:
\begin{equation*}
d_H(Ax, Ay) \leq \tau(A) \cdot d_H(x, y).
\end{equation*}
Since $K_{ij} = \exp(-C_{ij}/\varepsilon) > 0$ for all $i, j$, the kernel $K$ has finite projective diameter $\Delta(K) < \infty$, yielding $\tau(K) < 1$. By symmetry, $\Delta(K^\top) = \Delta(K)$, hence $\tau(K^\top) = \tau(K)$.

\textbf{Step 3: Isometry under diagonal scaling.}
For any fixed $z \in \mathbb{R}_{++}^n$, the map $\phi_z: x \mapsto z \oslash x$ is an isometry under $d_H$:
\begin{equation*}
d_H(z \oslash x, \, z \oslash y) = d_H(x, y).
\end{equation*}
This follows from: $(z_i/x_i)(y_j/z_j) / [(z_j/x_j)(y_i/z_i)] = (x_j y_i)/(x_i y_j)$.

\textbf{Step 4: Contraction of the Sinkhorn map.}
Define the half-iteration operators $T_r: v \mapsto r \oslash (Kv)$ and $T_c: u \mapsto c \oslash (K^\top u)$. For iterates $v, v' \in \mathbb{R}_{++}^m$, let $u = T_r(v)$ and $u' = T_r(v')$. Applying the isometry property and Birkhoff's theorem:
\begin{equation*}
d_H(u, u') = d_H(Kv, Kv') \leq \tau(K) \cdot d_H(v, v').
\end{equation*}
Similarly, for the second step: $d_H(T_c(u), T_c(u')) \leq \tau(K^\top) \cdot d_H(u, u')$. Composing both inequalities for one full iteration $T = T_c \circ T_r$:
\begin{equation*}
d_H(T(v), T(v')) \leq \tau(K)^2 \cdot d_H(v, v').
\end{equation*}
Since $\kappa := \tau(K)^2 < 1$, the map $T$ is a strict contraction.

\textbf{Step 5: Linear convergence.}
By Theorem~\ref{thm:sinkhorn}, the optimal scaling vectors $(u^\star, v^\star)$ exist and are unique up to multiplicative constant. Applying the Banach fixed-point theorem:
\begin{equation*}
d_H(v^{(t)}, v^\star) \leq \kappa^t \cdot d_H(v^{(0)}, v^\star).
\end{equation*}
The Hilbert metric convergence implies convergence of $\Pi^{(t)} = \mathrm{diag}(u^{(t)}) K \mathrm{diag}(v^{(t)})$ to $\Pi^\star$ in standard norms at rate $O(\kappa^t)$.
\end{proof}

\end{document}